\documentclass[sn-basic]{sn-jnl}

\usepackage{xspace}
\usepackage{bm}
\usepackage{mathtools}
\usepackage{amsmath,amsfonts}
\usepackage{amsthm}
\usepackage{xcolor}
\usepackage{subfig}
\usepackage{graphicx}
\usepackage{natbib}

\jyear{2021}%

\theoremstyle{thmstyleone}%
\newtheorem{theorem}{Theorem}
\newtheorem{proposition}[theorem]{Proposition}%

\theoremstyle{thmstyletwo}%
\newtheorem{remark}{Remark}%

\theoremstyle{thmstylethree}%

\DeclareRobustCommand{\eg}{e.g.,\@\xspace}                 
\DeclareRobustCommand{\ie}{i.e.,\@\xspace}                    
\DeclareRobustCommand{\wrt}{w.r.t.\@\xspace}

\DeclareRobustCommand{\quotes}[1]{``#1''}

\DeclareRobustCommand{\algname}{Linear Correlated Features Aggregation\@\xspace}     
\DeclareRobustCommand{\algnameshort}{LinCFA\@\xspace}     

\DeclareMathOperator{\E}{\mathbb{E}}

\algnewcommand\algorithmicforeach{\textbf{for each}}
\algdef{S}[FOR]{ForEach}[1]{\algorithmicforeach\ #1\ \algorithmicdo}

\newtheorem{lemma}{Lemma}

\begin{document}

\title[Interpretable Linear Dimensionality
Reduction]{Interpretable Linear Dimensionality Reduction based on Bias-Variance Analysis}

\author*[1]{\fnm{Paolo} \sur{Bonetti}}\email{paolo.bonetti@polimi.it}

\author[1]{\fnm{Alberto Maria} \sur{Metelli}}\email{ albertomaria.metelli@polimi.it}

\author[1]{\fnm{Marcello} \sur{Restelli}}\email{marcello.restelli@polimi.it}

\affil[1]{\orgdiv{DEIB}, \orgname{Politecnico di Milano}, \orgaddress{\city{Milan}, \postcode{20133}, \country{Italy}}}


\abstract{One of the central issues of several machine learning applications on real data is the choice of the input features. Ideally, the designer should select only the relevant, non-redundant features to preserve the complete information contained in the original dataset, with little collinearity among features and a smaller dimension. This procedure helps mitigate problems like overfitting and the curse of dimensionality, which arise when dealing with high-dimensional problems. On the other hand, it is not desirable to simply discard some features, since they may still contain information that can be exploited to improve results.\\
Instead, \emph{dimensionality reduction} techniques are designed to limit the number of features in a dataset by projecting them into a lower-dimensional space, possibly considering all the original features. However, the projected features resulting from the application of dimensionality reduction techniques are usually difficult to interpret.\\
In this paper, we seek to design a principled dimensionality reduction approach that maintains the interpretability of the resulting features. Specifically, we propose a bias-variance analysis for linear models and we leverage these theoretical results to design an algorithm, \emph{\algname} (\algnameshort), which aggregates groups of continuous features with their average if their correlation is \quotes{sufficiently large}. In this way, all features are considered, the dimensionality is reduced and the interpretability is preserved. Finally, we provide numerical validations of the proposed algorithm both on synthetic datasets to confirm the theoretical results and on real datasets to show some promising applications.}




\maketitle

\section{Introduction}\label{intro}

Dimensionality reduction plays a crucial role in applying Machine Learning (ML) techniques in real-world datasets~\citep{Sorzano2014}. Indeed, in a large variety of scenarios, data are high-dimensional with a large number of correlated features. For instance, \emph{financial} datasets are characterized by time series representing the trend of stocks in the financial market, and \emph{climatological} datasets include several highly-correlated features that, for example, represent temperature value at different points on the Earth. On the other hand, only a small subset of features is usually significant for learning a specific task, and it should be identified to train a well-performing ML algorithm. In particular, considering many redundant features boosts the model complexity, which increases its variance and the risk of overfitting~\citep{hastie2009}. Furthermore, when the number of features is high, and comparable with the number of samples, the available data become sparse, leading to poor performance (\emph{curse of dimensionality}~\citep{bishop2006}). For this reason, \emph{dimensionality reduction} and \emph{feature selection} techniques are usually applied. Feature selection~\citep{chandrashekar2014survey} focuses on choosing a subset of features important for learning the target following a specific criterion (e.g., the most correlated with the target, the ones that produce the highest validation score), discarding the others. On the other hand, dimensionality reduction methods~\citep{Sorzano2014} maintain all the features projecting them in a (much) lower dimensional space, producing new features that are linear or non-linear combinations of the original ones. Compared to feature selection, this latter approach has the advantage of reducing the dimensionality without discarding any feature and exploiting all of their contributions to the projections. Moreover, recalling that the variance of a sum of random variables is smaller than or equal to the original one, the features computed with linear dimensionality reduction have smaller variance. However, the reduced features might be less interpretable since they are linear combinations of the original ones with different coefficients.
\\ \ \\ 
In this paper, we propose a novel dimensionality reduction method that exploits the information of each feature, without discarding any of them, while preserving the interpretability of the resulting feature set. To this end, we \emph{aggregate} features through their average, and we propose a criterion that aggregates two features when it is beneficial in terms of the bias-variance tradeoff. Specifically, we focus on linear regression, assuming a linear relationship between the features and the target. In this context, the main idea of this work is to identify a group of \emph{aggregable} features and substitute them with their average. Intuitively, in linear settings, two features should be aggregated if their correlation is \emph{large enough}. We identify a theoretical threshold on the minimum correlation for which it is profitable to unify the two features. This threshold is the minimum correlation value between two features for which, comparing the two linear regression models before and after the aggregation, the variance decrease is larger than the increase of bias. 
\\ \ \\
Choosing the average to aggregate the features is to preserve interpretability (the resulting reduced feature is just the average of $k$ original features). Another advantage is that the variance of the average is smaller than the variance of the original features if they are not perfectly correlated. Indeed, assuming that we unify $k$ standardized features, the variance of their average becomes $var(\bar{X})=\frac{1}{k}+\frac{k-1}{k}\rho$, with $\rho$ being the average correlation of distinct features~\citep{jacod2004}. The main restriction of choosing the average to aggregate is that we will only consider continuous features since the mean is not well-defined for categorical features. Moreover, it would be improper to evaluate the mean between heterogeneous features: interpretability is preserved only if the aggregation is meaningful. 
Indeed, the practical motivation behind this work relates to spatially distributed climatic datasets, where thousands of measurements of the same variable (\eg temperature) are available at different points on the Earth. In this context, feeding a supervised model directly with all the variables would be impractical, both in terms of the curse of dimensionality and collinearity, since measurements that are close in space have a large correlation. Aggregating these variables with their mean leads to a reduced dataset with less collinearity and preserving interpretability, since each variable represents the mean of the measurement of the considered quantity over an area identified in a data-driven way, with theoretical guarantees.\\ 
Another issue may arise when considering features with a different unit of measurement or scale, for this reason we will consider standardized variables.
\\ \ \\
\textbf{Outline:}
The paper is structured as follows. In Section \ref{sec:dimRed}, we formally define the problem, and we provide a brief overview of the main dimensionality reduction methods. Section \ref{sec:method} introduces the methodology that will be followed throughout the paper. In Section \ref{sec:2Dalgorithm}, the main theoretical result is presented for the bivariate setting, which is then generalized to $D$ dimensions in Section \ref{sec:3andD}. Finally, in Section \ref{sec:appl}, the proposed algorithm, \emph{\algname} (\algnameshort), is applied to synthetic and real-world datasets to experimentally confirm the result and lead to the conclusions of Section \ref{sec:conclusion}. The paper is accompanied by supplementary material. Specifically, Appendix \ref{app:twoDim} contains the proofs and technical results of the bivariate case that are not reported in the main paper, Appendix \ref{app:addRes} shows an additional finite-samples bivariate analysis, Appendix \ref{subsec:confInt} elaborates on the bivariate results to be composed only of theoretical or empirical quantities, Appendix \ref{app:proofs3D} contains the proofs and technical results of the three-dimensional setting, and Appendix \ref{app:exp} presents in more details the experiments performed.

\section{Preliminaries}\label{sec:dimRed}
In this section, we introduce the notation and assumptions employed in the paper (Section~\ref{sec:notation}) and we survey the main related works (Section~\ref{sec:relatedWorks}).

\subsection{Notation and Assumptions}\label{sec:notation}

Let $(X,Y)$ be random variables with joint probability distribution $P_{X,Y}$, where $X\in \mathbb{R}^D$ is the $D$-dimensional vector of features and $Y\in \mathbb{R}$ is the scalar target of a supervised learning regression problem.
Given $N$ data sampled from the distribution $P_{X,Y}$, we denote the corresponding feature matrix as $\mathbf{X}\in \mathbb{R}^{N\times D}$ and the target vector as $\mathbf{Y}\in \mathbb{R}^{N}$. Each element of the random vector $X$ is denoted with $x_i$ and it is called a \emph{feature} of the ML problem. We denote as $y$ the scalar target random variable and with $\sigma^2_y$ and $\hat{\sigma}^2_y$ its variance and sample variance. For each pair of random variables $a,b$ we denote with $\sigma^2_{a}$, $cov(a,b)$ and $\rho_{a,b}$ respectively the variance of the random variable $a$ and its covariance and correlation with the random variable $b$. Their estimators are $\hat{\sigma}^2_{a}$, $\hat{cov}(a,b)$ and $\hat{\rho}_{a,b}$. 
Finally, the expected value and the variance operators applied on a function $f(a)$ of a random variable $a$ \wrt its distribution are denoted with $\mathbb{E}_a[f(a)]$ and $var_a(f(a))$.

A dimensionality reduction method can be seen as a function $\bm{\phi}: \mathbb{R}^{N \times D} \rightarrow \mathbb{R}^{N \times d} $, mapping the original feature matrix $\mathbf{X}$  with dimensionality $D$ into a reduced dataset $\mathbf{U} = \bm{\phi}(\mathbf{X}) \in \mathbb{R}^{N\times d}$ with $d<D$. The goal of this projection is to reduce the (possibly huge) dimensionality of the original dataset while keeping as much information as possible in the reduced dataset. This is usually done by preserving a distance (e.g., Euclidean, geodesic) or the probability of a point to have the same neighbours after the projection~\citep{zaki2014}.

In this paper, we assume a \emph{linear} dependence between the features $X$ and the target $Y$, \ie $Y = w^{T} X + \epsilon$, where $\epsilon$ is a zero-mean noise, independent of $X$, and $w \in \mathbb{R}^{D}$ is the weight vector. Without loss of generality, the expected value of each feature is assumed to be zero, \ie $\mathbb{E}[x_i]=\mathbb{E}[Y]=0\ \forall i\in \{1,\dots,D\}$. Finally, we consider linear regression as ML method: the $i$-th estimated coefficient is denoted with $\hat{w}_i$, the estimated noise with $\hat{\epsilon}$ and the predicted (scalar) target with $\hat{y}$.

\subsection{Existing methods}\label{sec:relatedWorks}
\subsubsection{Unsupervised Dimensionality Reduction} Classical dimensionality reduction methods can be considered as \emph{unsupervised} learning techniques which, in general, do not take into account the target, but they focus on projecting the dataset $\mathbf{X}$ through the minimization of a given loss.

The most popular unsupervised linear dimensionality reduction technique is Principal Components Analysis (PCA)~\citep{pearson1901,Hotelling1933}, a linear method that embeds the data into a linear subspace of dimension $d$ describing as much as possible the variance in the original dataset. One of the main difficulties of applying PCA in real problems is that it performs linear combinations of possibly all the $D$ features, usually with different coefficients, losing the interpretability of each principal component and suffering the curse of dimensionality. To overcome this issue, there exist some variants like svPCA~\citep{Ulfarsson2011}, which forces most of the weights of the projection to be zero. This contrasts with the approach proposed in this paper, which aims to preserve interpretability while exploiting the information yielded by each feature. \\
There exist several variants to overcome different issues of PCA (e.g., out-of-sample generalization, linearity, sensitivity to outliers) and other methods that approach the problem from a different perspective (e.g., generative approach with Factor Analysis, independence-based approach with Independent Component Analysis, matrix factorization with SVD), an extensive overview can be found in~\citep{Sorzano2014}. A broader overview of linear dimensionality reduction techniques can be found in~\citep{cunningham2015}.
Specifically, SVD~\citep{Golub1970} leads to the same result of PCA from an algebraic perspective through matrix decomposition. Factor analysis~\citep{Thurstone1931} assumes that the features are generated from a smaller set of latent variables, called factors, and tries to identify them by looking at the covariance matrix. Both PCA and Factor Analysis can reduce through rotations the number of features that are combined for each reduced component to improve the interpretability, but their coefficients can still be different and hard to interpret. Finally, Independent Component Analysis~\citep{Hyvarinen1999} is an information theory approach that looks for independent components (not only uncorrelated as PCA) that are not constrained to be orthogonal. This method is more focused on splitting different signals mixed between features than on reducing their dimensionality, which can be done as a subsequent step with feature selection, which would be simplified from the fact that the new features are independent.

Differently from the linear nature of PCA, many non-linear approaches exist, following the idea that the data can be projected onto non-linear manifolds. Some of them optimize a convex objective function (usually solvable through a generalized eigenproblem) trying to preserve global similarity of data (e.g., Isomap~\citep{Tenenbaum2000}, Kernel PCA~\citep{shawe2004}, Kernel Entropy Component Analysis~\citep{Jenssen2010}, MVU~\citep{Weinberger2004}, Diffusion Maps~\citep{Lafon2006}) or local similarity of data (LLE~\citep{Roweis2000}, Laplacian Eigenmaps~\citep{Belkin2001}, LTSA~\citep{Zhang2004}). Other methods optimize a non-convex objective function with the purpose of rescaling Euclidean distance (Sammon Mapping~\citep{Sammon1969}) introducing more complex structures like neural networks (Multilayer Autoencoders~\citep{Hinton2006}) or aligning mixtures of models (LLC~\citep{teh2002}). Since in this paper we assume linearity, non-linear techniques for dimensionality reduction cannot outperform traditional linear techniques such as PCA and its variants in linear contexts~\citep{vandermaaten2009} or with real data~\citep{Espadoto2021}, therefore they will not be compared to the proposed method. On the contrary, classical PCA is one of the most applied linear unsupervised dimensionality reduction techniques in ML applications, therefore, it will be compared with the proposed algorithm \algnameshort in the experimental section.
\\ 

\subsubsection{Supervised Dimensionality Reduction} Supervised dimensionality reduction is a less-known but powerful approach when the main goal is to perform classification or regression rather than learn a data projection into a lower dimensional space. The methods of this subfield are usually based on classical unsupervised dimensionality reduction, adding the regression or classification loss in the optimization phase. In this way, the reduced dataset $\mathbf{U}$ is the specific projection that allows maximizing the performance of the considered supervised problem. This is usually done in classification settings, minimizing the distance within the same class and maximizing
the distance between different classes in the same fashion as Linear Discriminant Analysis~\citep{fisher1936}. The other possible approach is to directly integrate the loss function for classification or regression. Following the taxonomy presented in~\citep{chao2019}, these supervised approaches can be divided into PCA-based, NMF-based (mostly linear), and manifold-based (mostly non-linear).

A well-known PCA-based algorithm is Supervised PCA. The most straightforward approach of this kind has been proposed in~\citep{Bair2006}, which is a heuristic that applies classical PCA only to the subset of features mostly related to the target. A more advanced approach can be found in~\citep{Barshan2011}, where the original dataset is orthogonally projected onto a space where the features are uncorrelated, simultaneously maximizing the dependency between the reduced dataset and the target by exploiting Hilbert–Schmidt independence criterion.
The goal of Supervised PCA is similar to that of the algorithm proposed in this paper. The main difference is that we are not looking for an orthogonal projection, but we aggregate features by computing their means (thus, two projected features can be correlated) to preserve interpretability. Many variants of Supervised PCA exist, e.g., to make it a non-linear projection or to make it able to handle missing values~\citep{yu2006}. Since it is defined in the same context (linear) and has the same final purpose (minimize the mean squared regression error), supervised-PCA will be compared with the approach proposed by this paper in the experimental section.
NMF-based algorithms~\citep{Jing2012,lu2017} have better interpretability than PCA-based, but they focus on the non-negativity property of features, which is not a general property of linear problems. Manifold-based methods~\citep{ribeiro2008,Zhang2018,Zhang2009,raducanu2012}, on the other hand, perform non-linear projections with higher computational costs. Therefore, both families of techniques will not be considered in this linear context.

\section{Proposed Methodology}
\label{sec:method}
In this section, we introduce the proposed dimensionality reduction algorithm, named \emph{\algname} (\algnameshort), from a general perspective. The approach is based on the following simple idea. Starting from the features $x_i$ of the $D$-dimensional vector $X$, we build the aggregated features $u_k$ of the $d$-dimensional vector $U$. The dimensionality reduction function $\bm{\phi}$ is fully determined by a partition  $\bm{\mathcal{P}}=\{\mathcal{P}_1,\dots,\mathcal{P}_d\}$ of the set of features $\{x_1,\dots,x_D\}$. In particular,  each feature $x_i$ is assigned to a set $\mathcal{P}_k\in \bm{\mathcal{P}}$ and each feature $u_k$ is computed as the average of the features in the $k$-th set of $\bm{\mathcal{P}}$:

\begin{align}
    u_k = \frac{1}{\lvert\mathcal{P}_k\rvert} \sum_{i \in \mathcal{P}_k} x_i.
\end{align}

In the following sections, we will focus on finding theoretical guarantees to determine how to build the partition $\bm{\mathcal{P}}$. Intuitively, two features will belong to the same element of the partition $\bm{\mathcal{P}}$ if their correlation is larger than a threshold. This threshold is formalized as the minimum correlation for which the Mean Squared Error (\textit{MSE}) of the regression with a single aggregated feature (i.e., the average) is not worse than the \textit{MSE} with the two separated features.\footnote{For this reason, the approach can be considered \emph{supervised}.} In particular, it is possible to decompose the \textit{MSE} as follows (bias-variance decomposition~\citep{hastie2009}):
\begin{equation}\label{eq:BiasVarDec}
\begin{gathered}
	      \underbrace{\mathbb{E}_{x,y,\mathcal{T}}[(h_\mathcal{T}(x)-y)^2]}_{\text{MSE}}
	      = \underbrace{\mathbb{E}_{x,\mathcal{T}}[(h_\mathcal{T}(x)-\bar{h}(x))^2]}_{\text{variance}}
	     \\\qquad+\underbrace{\mathbb{E}_{x}[(\bar{h}(x)-\bar{y}(x))^2]}_{\text{bias}}
	      +\underbrace{\mathbb{E}_{x,y}[(\bar{y}(x)-y)^2]}_{\text{noise}},
	\end{gathered}
\end{equation}
where $x,y$ are the features and the target of a test sample, $\mathcal{T}$ is the training set, $h_\mathcal{T}(\cdot)$ is the ML model trained on dataset $\mathcal{T}$, $\bar{h}(\cdot)$ is its expected value \wrt\ the training set $\mathcal{T}$ and $\bar{y}$ is the expected value of the test output target $y$ \wrt\ the input features $x$.
Decreasing model complexity leads to a decrease in variance and an increase in bias. Therefore, in the analysis, we will compare these two variations and identify a threshold as the minimum value of correlation for which, after the aggregation, the decrease of variance is greater or equal than the increase of bias, so that the \textit{MSE} will be greater or equal than the original one.

\section{Two-dimensional Analysis}
\label{sec:2Dalgorithm}
This section introduces the theoretical analysis, performed in the bivariate setting, that identifies the minimum value of the correlation between the two features for which it is convenient to aggregate them with their mean. In particular, Subsection \ref{subsec:setting} introduces the assumptions under which the analysis is performed. Subsection \ref{subsec:var} computes the amount of variance decreased when performing the aggregation. Then, Subsection \ref{subsec:bias} evaluates the amount of bias increased due to the aggregation. Finally, Subsection \ref{subsec:corr2D} combines the two results identifying the minimum amount of correlation for which it is profitable to aggregate the two features. In addition, Appendix \ref{app:twoDim} contains the proofs and technical results that are not reported in the main paper, Appendix \ref{app:addRes} includes an additional finite-sample analysis, and Appendix \ref{subsec:confInt} computes confidence intervals that allow stating the results with only theoretical or empirical quantities.

\subsection{Setting}\label{subsec:setting}
In the two-dimensional case ($D=2$), we consider the relationship between the two features $x_1$, $x_2$ and the target $y$ to be linear and affected by Gaussian noise: $y=w_1x_1+w_2x_2+\epsilon$, with $\epsilon\sim \mathcal{N}(0,\sigma^2)$. As usually done in linear regression~\citep{johnson2007}, we assume the training dataset $\mathbf{X}$ to be known. Moreover, recalling the zero-mean assumption ($\E[x_1]=\E[x_2]=0$), it follows $\mathbb{E}[y]=w_1\mathbb{E}[x_1]+w_2\mathbb{E}[x_2]=0$ and $\sigma^2_y = \sigma^2$.

We compare the performance (in terms of bias and variance) of the two-dimensional linear regression $\hat{y}=\hat{w}_1x_1+\hat{w}_2x_2$ with the one-dimensional linear regression, which takes as input the average between the two features $\hat{y}=\hat{w}\frac{x_1+x_2}{2}=\hat{w}\Bar{x}$. As a result of this analysis, we will define conditions under which aggregating features $x_1$ and $x_2$ in the feature $\Bar{x}$ is convenient.

\subsection{Variance Analysis}\label{subsec:var}
In this subsection, we compare the variance of the two models with both an asymptotic and a finite-samples analysis. Since the two-dimensional model estimates two coefficients, it is expected to have a larger variance. Instead, aggregating the two features reduces the variance of the model.

\subsubsection{Variance of the estimators}\label{subsubsec:varEst}
A quantity, necessary to compute the variance of the models that will be compared throughout this subsection, is the covariance matrix of the vector $\hat{w}$ of the estimated regression coefficients \wrt\ the training set. Given the training features $\mathbf{X}$, a known result in a general linear problem with $n$ samples and $D$ features~\citep{johnson2007} (see Appendix \ref{app:twoDim} for the computations) is:
\begin{equation}
\label{eq:VarEst}
var_{\mathcal{T}}(\hat{w}\lvert\mathbf{X}) = (\mathbf{X}^T\mathbf{X})^{-1}\sigma^2.
\end{equation}
The following lemma shows the variance of the weights for the two specific models that we are comparing.
\begin{lemma}\label{lemma:varianceEstimator}
Let the real model be linear with respect to the features $x_1$ and $x_2$ ($y=w_1x_1+w_2x_2+\epsilon$). In the one-dimensional case $\hat{y}=\hat{w}\bar{x}$, we have:
\begin{equation}
 var_{\mathcal{T}}(\hat{w}\lvert\mathbf{X}) = \frac{\sigma^2}{(n-1)\hat{\sigma}^2_{\bar{x}}}.
\label{eq:varEst1d}
\end{equation}
In the two-dimensional case $\hat{y}=\hat{w}_1x_1+\hat{w}_2x_2$, we have:
\begin{equation}
\begin{aligned}
var_{\mathcal{T}}(\hat{w}\lvert\mathbf{X}) & = \frac{\sigma^2}{(n-1)(\hat{\sigma}^2_{x_1}\hat{\sigma}^2_{x_2} - \hat{cov}(x_1,x_2)^2)} \\ & \quad \times \begin{bmatrix} \hat{\sigma}^2_{x_2} & -\hat{cov}(x_1,x_2) \\ -\hat{cov}(x_1,x_2) & \hat{\sigma}^2_{x_1} \end{bmatrix}.
\label{eq:varEst2d}
\end{aligned}
\end{equation}
\end{lemma}
\begin{proof}
The proof of the two results follows from Equation \eqref{eq:VarEst}, see Appendix
\ref{app:twoDim} for the computations. 
\end{proof}

\subsubsection{Variance of the model}\label{subsubsec:varMod}
Recalling the general definition of variance of the model from Equation \eqref{eq:BiasVarDec}, in the specific case of linear regression it becomes:
\begin{equation}
\E_{x,{\mathcal{T}}}[(h_{\mathcal{T}}(x)-\Bar{h}(x))^2] = \E_{x,{\mathcal{T}}}[(\hat{w}^T x-\E_{\mathcal{T}}[\hat{w}^Tx])^2].
\label{eq:modelVar}
\end{equation}

The following result shows the variance of the two specific models (univariate and bivariate) considered in this section.

\begin{theorem}
\label{thm:variance2D}
Let the real model be linear with respect to the two features $x_1$ and $x_2$ ($y=w_1x_1+w_2x_2+\epsilon$). Then, in the one dimensional case $y=\hat{w}\frac{x_1+x_2}{2}=\hat{w}\bar{x}$, we have:
\begin{equation}
\label{eq:variance1D}
\begin{gathered}
\E_{x,{\mathcal{T}}}[(h_{\mathcal{T}}(x)-\Bar{h}(x))^2\lvert \mathbf{X}] = \sigma_{x_1+x_2}^2\frac{\sigma^2}{(n-1)\hat{\sigma}^2_{x_1+x_2}}.
\end{gathered}
\end{equation}
In the two dimensional case $y=\hat{w}_1x_1+\hat{w}_2x_2$, we have:
\begin{equation}
\label{eq:variance2D}
\begin{aligned}
& \E_{x,{\mathcal{T}}}[(h_{\mathcal{T}}(x)-\Bar{h}(x))^2\lvert \mathbf{X}] \\ 
& \quad = \frac{\sigma^2(\sigma^2_{x_1}\hat{\sigma}^2_{x_2}+\sigma^2_{x_2}\hat{\sigma}^2_{x_1}-2cov(x_1,x_2)\hat{cov}(x_1,x_2))}{(n-1)(\hat{\sigma}^2_{x_1}\hat{\sigma}^2_{x_2} - \hat{cov}(x_1,x_2)^2)}. 
\end{aligned}
\end{equation}

\end{theorem}

\begin{proof}
The proof combines the results of Lemma \ref{lemma:varianceEstimator} with the definition of variance for a linear model given in Equation \eqref{eq:modelVar}. The detailed proof can be found in Appendix \ref{app:twoDim}.
\end{proof}

\subsubsection{Comparisons}\label{subsubsec:varComp}
In this subsection, the difference between the variance of the linear regression with two features $x_1$ and $x_2$ and the variance of the linear regression with one feature $\bar{x}=\frac{x_1+x_2}{2}$ is shown. We will prove that, as expected, this difference is positive and it represents the reduction of variance when substituting a two-dimensional random vector with the average of its components.

First, the \emph{asymptotic} analysis is performed, obtaining a result that can be applied with good approximation when a large number of samples $n$ is available. Then, the analysis is repeated in the \emph{finite-samples} setting, with an additional assumption on the variance and sample variance of the features $x_1$ and $x_2$, that simplify the computations.\footnote{The assumption that we will introduce for the finite-samples setting might be restrictive.  However, it allows simplifying the computations. A more general finite-sample analysis has also been performed, only assuming unitary variances. This more general analysis leads to more convolute expressions and for this reason it is reported in Appendix \ref{app:addRes}.} 

\paragraph*{Case I: asymptotic analysis}
In the limit of the number of samples $n$ of the training dataset $\mathcal{T}$ approaching infinity, the estimators that we are considering are \emph{consistent}, \ie they converge in probability to the real values of the parameters (\eg $\mathrm{plim}_{n\to \infty}\hat{\sigma}^2_{x_1}=\sigma^2_{x_1}$). Therefore the following result can be proved.
\begin{theorem}\label{thm:asymVar2D}
If the number of samples $n$ tends to infinity, let $\Delta_{var}^{n\to \infty}$ be the difference between the variance of the two-dimensional and the one-dimensional linear models, it is equal to:
\begin{equation}\label{eq:asymDiffVar}
\Delta_{var}^{n\to \infty}=\frac{\sigma^2}{n-1} \ge 0, 
\end{equation}
that is a positive quantity and tends to zero when the number of samples tends to infinity.
\end{theorem}
\begin{proof}
The result follows from the difference between Equation \ref{eq:variance2D} and \ref{eq:variance1D}, exploiting the consistency of the estimators.
\end{proof}
 
\paragraph*{Case II: finite-samples analysis with equal variance and sample variance}
For the finite-samples analysis, we add the following assumption to simplify the computations:
\begin{equation}
\label{eq:sameVar}
    \begin{cases} \sigma_{x_1}=\sigma_{x_2}\eqqcolon\sigma_x \\ \hat{\sigma}_{x_1}=\hat{\sigma}_{x_2}\eqqcolon\hat{\sigma}_x \end{cases}.
\end{equation}
\begin{theorem}\label{thm:diffFinite}
If the conditions of Equation \eqref{eq:sameVar} hold, let $\Delta_{var}$ be the difference between the variance of the two-dimensional and the one-dimensional linear models, it is always non-negative and it is equal to: \begin{equation}\label{eq:finDiffVar}
\Delta_{var}=\frac{\sigma^2}{n-1}\cdot\frac{\sigma^2_x(1-\rho_{x_1,x_2})}{\hat{\sigma}^2_x(1-\hat{\rho}_{x_1,x_2})}.
\end{equation}
\end{theorem}

\begin{proof}
The proof starts again from the variances of the two models found in Theorem \ref{thm:variance2D} and it performs algebraic computations exploiting the assumption stated in Equation \eqref{eq:sameVar}. All the steps can be found in Appendix \ref{app:twoDim}.
\end{proof}

\begin{remark}
When the number of samples $n$ tends to infinity, the result of Equation~\eqref{eq:finDiffVar} reduces to the asymptotic case, as in Equation \eqref{eq:asymDiffVar}.
\end{remark}

\begin{remark}
The quantities found in Theorem \ref{thm:asymVar2D} and \ref{thm:diffFinite} are always non-negative, meaning that the variance of the two-dimensional case is always greater or equal than the corresponding one-dimensional version, as expected.
\end{remark}

\subsection{Bias Analysis}\label{subsec:bias}
In this subsection, we compare the (squared) bias of the two models under examination with both an asymptotic and a finite-samples analysis, as done in the previous subsection for the variance. Since the two-dimensional model corresponds to a larger hypothesis space it is expected to have a lower bias w.r.t. the one-dimensional.

The procedure to derive the difference between biases is similar to the one followed for the variance. The first step is to compute the expected value \wrt the training set $\mathcal{T}$ of the vector $\hat{w}$ of the regression coefficients estimates, given the training features $\mathbf{X}$. This is used to compute the bias of the models. In particular, in Equation \eqref{eq:BiasVarDec}, we defined the (squared) bias as follows:
\begin{equation}
\E_x[(\Bar{h}(x)-\bar{y})^2]= \E_x[(\E_{\mathcal{T}}[h(x)]-\E_{y\lvert x}[y])^2].
\label{eq:modelBias}
\end{equation}
Starting from this definition, the bias of the one-dimensional case $\hat{y}=\hat{w}\bar{x}$ is computed. Moreover, for the two dimensional case $y=\hat{w}_1x_1+\hat{w}_2x_2$ the model is clearly unbiased.
Detailed computations can be found in Appendix \ref{app:twoDim}.

After the derivation of the bias of the models, the same asymptotic and finite-samples analysis performed on the variance is repeated in this section for the (squared) bias. Since the two-dimensional model is unbiased, we can conclude that the increase of the bias component of the loss, when the two features are substitute by their mean, is equal to the bias of the one-dimensional model.

\paragraph*{Case I: asymptotic analysis}
When the number of samples $n$ of the training dataset $\mathcal{T}$ approaches infinity, recalling that the estimators considered converge in probability to the expected values of the parameters, the following result holds.
\begin{theorem}\label{thm:asym2dBias}
If the number of samples $n$ tends to infinity, let $\Delta_{bias}^{n\to \infty}$ be the difference between the bias of the one-dimensional and the two-dimensional models, it is equal to:
\begin{align}\label{eq:asymDiffBias}
\Delta_{bias}^{n\to \infty}&=\frac{\sigma^2_{x_1}\sigma^2_{x_2}(1-\rho_{x_1,x_2}^2)(w_1-w_2)^2}{\sigma^2_{x_1+x_2}}\\&=\frac{(1-\rho_{x_1,x_2})(w_1-w_2)^2}{2}, 
\end{align}
where the second equality holds if $\sigma_{x_1}=\sigma_{x_2}=1$.
\end{theorem}
\begin{proof}
The proof starts from the bias of the two models computed in Appendix \ref{app:twoDim} and exploits the fact that
in the limit $n \rightarrow \infty$, it is possible to substitute every sample estimator with the real quantity of the parameters because they are consistent estimators. Details can be found in Appendix \ref{app:twoDim}.
\end{proof}

\paragraph*{Case II: finite-samples analysis with equal variance and sample variance}
In the finite-samples case, we provide the same analysis performed for variance, \ie with the assumptions of Equation \eqref{eq:sameVar}.

\begin{theorem}\label{thm:diffFiniteBias}
If the conditions of Equation \eqref{eq:sameVar} hold, let $\Delta_{bias}$ be the difference between the (squared) bias of the one-dimensional and the two-dimensional linear models, then it has value: 
\begin{equation}\label{eq:finDiffBias}
\Delta_{bias}=\frac{\sigma^2_x(1-\rho_{x_1,x_2})(w_1-w_2)^2}{2}.
\end{equation}
\end{theorem}

\begin{proof}
The proof starts from the bias of the two models and performs algebraic computations exploiting the assumptions of Equation \eqref{eq:sameVar}. All the steps can be found in Appendix \ref{app:twoDim}.
\end{proof}

\begin{remark}
When the number of samples $n$ tends to infinity, the result in Equation~\eqref{eq:finDiffBias} reduces to the asymptotic case as in Theorem \ref{thm:asym2dBias}.
\end{remark}

\begin{remark}
Some observations are in order:
\begin{itemize}
    \item As expected, the quantities found in Theorem \ref{thm:asym2dBias}, \ref{thm:diffFiniteBias} are always non-negative, since the hypothesis space of the univariate model is a subset of the one of the bivariate model.
    \item We observe that $\Delta_{bias}=0$ if $\rho_{x_1,x_2}=1$. Indeed, when the two variables are perfectly (positively) correlated their coefficients in the linear regression are equal, therefore there is no loss of information in their aggregation.
    \item Finally, when the two regression coefficients are equal $w_1=w_2$ there is no increase of bias due to the aggregation, since it is enough to learn a single coefficient $\bar{w}$ to have the same performance of the bivariate model.
\end{itemize}
\end{remark}

\subsection{Correlation Threshold}\label{subsec:corr2D}
This subsection concludes the analysis with two features by comparing the reduction of variance with the increase of bias when aggregating the two features $x_1$ and $x_2$ with their average $\bar{x}=\frac{x_1+x_2}{2}$. In conclusion, the result shows when it is convenient to aggregate the two features with their mean, in terms of mean squared error.

Considering the asymptotic case, the following theorem compares bias and variance of the models.
\begin{theorem}\label{thm:asympBound2D}
When the number of samples $n$ tends to infinity and the relationship between the features and the target is linear with Gaussian noise, the decrease of variance is greater than the increase of (squared) bias when the two features $x_1$ and $x_2$ are aggregated with their average if and only if:
\begin{equation}\label{eq:2DasymBound}
    \rho^2_{x_1,x_2} \geq 1-\frac{\sigma^2\sigma^2_{x_1+x_2}}{(n-1)\sigma^2_{x_1}\sigma^2_{x_2}(w_1-w_2)^2},
\end{equation}
that, for $\sigma_{x_1}=\sigma_{x_2}=1$ becomes:
    \begin{equation}\label{eq:2DasymBuondRed}
        \rho_{x_1,x_2} \geq 1-\frac{2\sigma^2}{(n-1)(w_1-w_2)^2}.
    \end{equation}
\end{theorem}
\begin{proof}
Computing the difference between Equation \eqref{eq:asymDiffVar} and \eqref{eq:asymDiffBias} the result follows.
\end{proof}

In the finite-samples setting, with the additional assumptions of Equation \eqref{eq:sameVar}, the following theorem shows the result of the comparison between bias and variance of the two models.

\begin{theorem}\label{thm:finiteBound2D}
Let the variance and sample variance of the features $x_1$ and $x_2$ be equal (Equation \eqref{eq:sameVar}) and the relationship between the features and then target be linear with Gaussian noise. The decrease of variance is greater than the increase of (squared) bias when the two features $x_1$ and $x_2$ are aggregated with their average if and only if:
\begin{equation}\label{eq:2DfinBound}
    \hat{\rho}_{x_1,x_2} \geq 1-\frac{2\sigma^2}{(n-1)\hat{\sigma}^2_x(w_1-w_2)^2},
\end{equation}
that, for $\hat{\sigma}_x=1$ becomes:
    \begin{equation}\label{eq:2DfinBuondRed}
        \hat{\rho}_{x_1,x_2} \geq 1-\frac{2\sigma^2}{(n-1)(w_1-w_2)^2}.
    \end{equation}
\end{theorem}
\begin{proof}
Computing the difference between Equation \eqref{eq:finDiffVar} and \eqref{eq:finDiffBias} the result follows.
\end{proof}

\begin{remark}
The results of Theorem \ref{thm:asympBound2D} and \ref{thm:finiteBound2D} comply with the intuition that, in a linear setting with two features, they should be aggregated if their correlation is \emph{large enough}.
\end{remark}

\begin{remark}
Theorem \ref{thm:asympBound2D} and \ref{thm:finiteBound2D} with unitary sample variances produce the same threshold both in the finite and the asymptotic settings.
\end{remark}

In conclusion, the thresholds found in Theorem \ref{thm:asympBound2D} and \ref{thm:finiteBound2D} show that it is profitable in terms of $MSE$ to aggregate two variables in a bivariate linear setting with Gaussian noise if: 
\begin{itemize}
    \item the variance of the noise $\sigma^2$ is large, which means that the process is noisy and the variance should be reduced;
    \item the number of samples $n$ is small, indeed in this case there is little knowledge about the actual model, therefore it is better to learn one parameter rather than two;
    \item the difference between the two coefficients  $w_1-w_2$ is small, which implies that they are similar, and learning a single coefficient introduces a little loss of information.
\end{itemize}

\section{Generalization: three-dimensional and D-dimensional analysis}\label{sec:3andD}
In the previous section, we focused on aggregating two features in a bivariate setting. In this section, we extend that approach to three features. Starting from the related results, we will straightforwardly extend them to a general problem with $D$ features. Because of the complexity of the computations, we focus on asymptotic analysis only. After the analysis, we conclude this section with the main algorithm proposed in this paper: \emph{\algname} (\algnameshort).

\subsection{Three-dimensional case}\label{subsec:threeDim}

In the three-dimensional case ($D=3$), we consider the relationship between the three features and the target to be linear with Gaussian noise: $y=w_1x_1+w_2x_2+w_3x_3+\epsilon$, $\epsilon\sim \mathcal{N}(0,\sigma^2)$. In accordance with the previous analysis, we assume the training dataset $\mathbf{X}=[\mathbf{x_1}\ \mathbf{x_2}\ \mathbf{x_3}]$ to be known and recalling the zero-mean assumption ($\E[x_1]=\E[x_2]=\E[x_3]=0$) it follows $\mathbb{E}[y]=w_1\mathbb{E}[x_1]+w_2\mathbb{E}[x_2]+w_3\mathbb{E}[x_3]=0$, $\sigma^2_y = \sigma^2$.

In this setting and for the general $D$ dimensional setting of the next subsection, which will be a direct application of this, we compare the performance of the bivariate linear regression $\hat{y}=\hat{w}_ix_i+\hat{w}_jx_j$ of each pair of features $x_i,x_j$ with the univariate linear regression that considers their average $\hat{y}=\hat{w}\frac{x_i+x_j}{2}=\hat{w}\Bar{x}$, to decide whether it is convenient to aggregate them or not in terms of $MSE$. Indeed, extending the dimension from $D=2$ to a general dimension $D$ and comparing all the possible models where groups of variables are aggregated is combinatorial in the number of features and it would be impractical. Also, comparing the full $D$ dimensional regression model with the $D-1$ dimensional model where two variables are aggregated is impractical. Indeed, when the number of features is huge, in addition to a polynomial computational cost, both models suffer issues like the curse of dimensionality and risk of overfitting.

To simplify the exposition, for the theoretical analysis, we will consider $x_i=x_1,x_j=x_2$.

\subsubsection{Variance}
The following theorem shows the asymptotic difference of variance between the two considered linear regression models.

\begin{theorem}
\label{thm:3var}
Let $\mathbf{X}=[\mathbf{x_1}\ \mathbf{x_2}\ \mathbf{x_3}]$, $\hat{\mathbf{X}}=[\mathbf{x_1}\ \mathbf{x_2}]$ and $\bar{\mathbf{X}}=[\mathbf{\bar{x}}]$, with $\bar{x}=\frac{x_1+x_2}{2}$.
Then, for the one-dimensional linear regression $\hat{y}=\hat{w}\frac{x_1+x_2}{2}$, we have:
\begin{equation}\label{eq:var3D1}
var_{\mathcal{T}}(\hat{w}\lvert\mathbf{X}) = (\bar{\mathbf{X}}^T\bar{\mathbf{X}})^{-1}\sigma^2,
\end{equation}
and for the two-dimensional linear regression $\hat{y}=\hat{w}_1x_1+\hat{w}_2x_2$, we have:
\begin{equation}
    \label{eq:var3D2}
var_{\mathcal{T}}(\hat{w}\lvert\mathbf{X}) = (\hat{\mathbf{X}}^T\hat{\mathbf{X}})^{-1}\sigma^2.
\end{equation}
\end{theorem}
\begin{proof}
The results follow from the general expression of variance of the estimators from Equation \eqref{eq:VarEst} and substituting respectively $\bar{\mathbf{X}}$ and $\hat{\mathbf{X}}$ for the two considered models.
\end{proof}
\begin{remark}\label{rem:var3D}
Since the linear regression models are the same of the bivariate case, starting from the result of Theorem~\ref{thm:3var}, the variance of the estimators in the two cases remains the same of Lemma \ref{lemma:varianceEstimator} and the asymptotic difference of variances remains the one of Theorem \ref{thm:asymVar2D} ($\Delta_{var}^{n\to \infty}=\frac{\sigma^2}{(n-1)}$).
\end{remark}

\subsubsection{Bias}\label{subsubsec:bias}
This subsection introduces the asymptotic difference of bias of the two considered linear regression models in the three dimensional setting.

As in the bivariate setting, the first step is to calculate the bias for each of the two considered models.
In the asymptotic case, assuming unitary variances of the features $\sigma_{x_1}=\sigma_{x_2}=\sigma_{x_3}=1$, for the one-dimensional regression $\hat{y}=\hat{w}\frac{x_1+x_2}{2}$ and for the two-dimensional regression $\hat{y}=\hat{w_1}x_1+\hat{w_2}x_2$, the (squared) bias $\E_{x}[(\bar{h}(x)-\bar{y})^2]$ can be expressed with two functions, that will be denoted respectively with $\mathcal{F(\cdot)}$ and $\mathcal{G(\cdot)}$. These functions depend on the three features, their coefficients and their correlations. The exact expressions of the two biases and the related proofs can be found in Appendix \ref{app:proofs3D}.

For the extension to the $D$-dimensional case, it will be necessary to keep the feature $x_3$ having general variance $\sigma_{x_3}^2$. With little changes in the algebraic computations of the proof, the bias of the two models can be easily extended (see Appendix \ref{app:proofs3D} for the details).

From the results obtained, it is possible to compute the increase of bias due to the aggregation of the two variables $x_1,x_2$ with their average $\bar{x}=\frac{x_1+x_2}{2}$.

\begin{theorem}\label{thm:bias3dDiff}
In the asymptotic setting, let the relationship between the features and the target be linear with Gaussian noise. Assuming unitary variances of the features $\sigma_{x_1}=\sigma_{x_2}=\sigma_{x_3}=1$, the increase of bias due to the aggregation of the features $x_1$ and $x_2$ with their average is given by:
\begin{equation}
\begin{aligned}
\Delta_{Bias}^{n\to \infty} & = \frac{1}{2}(1-\rho_{x_1,x_2})(w_1-w_2)^2\\
& \quad +(w_1w_3-w_2w_3)(\rho_{x_1,x_3}-\rho_{x_2,x_3})\\
& \quad  +w_3^2\frac{(\rho_{x_1,x_3}-\rho_{x_2,x_3})^2}{2(1-\rho_{x_1,x_2})}.
\end{aligned}
\end{equation}
\end{theorem}
\begin{proof}
The result follows from the difference of the biases of the two models, after algebraic computations.
\end{proof}

\begin{remark}
Letting the feature $x_3$ having general variance $\sigma_{x_3}^2$, with little changes in the algebraic computations of the proof, the difference of biases is given by:
\begin{equation}
\begin{aligned}
\Delta_{Bias}^{n\to \infty} & = \frac{1}{2}(1-\rho_{x_1,x_2})(w_1-w_2)^2\\
&\quad+\sigma_{x_3}(w_1w_3-w_2w_3)(\rho_{x_1,x_3}-\rho_{x_2,x_3})\\
&\quad+w_3^2\sigma_{x_3}^2\frac{(\rho_{x_1,x_3}-\rho_{x_2,x_3})^2}{2(1-\rho_{x_1,x_2})}.
\end{aligned}
\end{equation}
\end{remark}

\subsubsection{Correlation threshold}
The result of the following theorem extends the result of Theorem \ref{thm:asympBound2D} for the three-dimensional setting.

\begin{theorem}\label{thm:3Dbound}
In the asymptotic setting, let the relationship between the features and the target be linear with Gaussian noise. Assuming unitary variances of the features $\sigma_{x_1}=\sigma_{x_2}=\sigma_{x_3}=1$, the decrease of variance is greater than the increase of (squared) bias due to the aggregation of the features $x_1$ and $x_2$ with their average if and only if:
{\medmuskip=1mu
\thinmuskip=1mu
\thickmuskip=1mu
\begin{equation}\label{eq:bound3D}
\begin{gathered}
1-(a-b)-\sqrt{a(a-2b)} \leq\rho_{x_1,x_2}\leq 1-(a-b)+\sqrt{a(a-2b)},\\
\text{with } \begin{cases}
a=\frac{\sigma^2}{(n-1)(w_1-w_2)^2}\\
b=\frac{(\rho_{x_1,x_3}-\rho_{x_2,x_3})w_3}{(w_1-w_2)}.
\end{cases}
\end{gathered}
\end{equation}
}
\end{theorem}
\begin{proof}
Recalling the asymptotic difference of variances from Remark \ref{rem:var3D}
and the difference of biases from Theorem \ref{thm:bias3dDiff}
the result follows after algebraic computations on the difference $\Delta_{var}^{n\to \infty} - \Delta_{Bias}^{n\to \infty} \geq 0.$
\end{proof}

\begin{remark}\label{rem:gen3D}
Equation \eqref{eq:bound3D} holds also in the case of generic variance $\sigma^2_{x_3}$ of the feature $x_3$, with the only difference that $b$ becomes:
\begin{equation}\label{eq:general3D}
    b=\frac{\sigma_{x_3}(\rho_{x_1,x_3}-\rho_{x_2,x_3})w_3}{(w_1-w_2)}.
\end{equation}
\end{remark}
\begin{remark}\label{rem:sameCorr}
The result obtained in this section with three features is more difficult to interpret than the bivariate one. However, if the two features $x_1$ and $x_2$ are uncorrelated with the third feature $x_3$ or they have the same correlation with it ($\rho_{x_1,x_3}=\rho_{x_2,x_3}$), then Equation \eqref{eq:bound3D} is equal to the one found in the bivariate asymptotic analysis (Equation \eqref{eq:2DasymBuondRed}).
\end{remark}
\begin{remark}
Since the analysis is asymptotic, the theoretical quantities in Equation \eqref{eq:bound3D} can be substituted with their consistent estimators when the number of samples $n$ is \emph{large}.
\end{remark}

\subsection{D-dimensional case}\label{sub:Ddim}

This last subsection of the analysis shows the generalization from three to $D$ dimensions. In particular, we assume the relationship between the $D$ features $x_1,...,x_D$ and the target to be linear with Gaussian noise $y=w_1x_1+...+w_Dx_D+\epsilon$, with $\epsilon\sim \mathcal{N}(0,\sigma^2)$. As done throughout the paper, we assume the training dataset $\mathbf{X}=[\mathbf{x_1}\ ...\ \mathbf{x_D}]$ to be known and from the zero-mean assumption $\mathbb{E}[y]=0$ and $\sigma^2_y = \sigma^2$.

As discussed for the three-dimensional case, we compare the performance (in terms of bias and variance) of the two-dimensional linear regression $\hat{y}=\hat{w}_ix_i+\hat{w}_jx_j$ with the one-dimensional linear regression $\hat{y}=\hat{w}\frac{x_i+x_j}{2}=\hat{w}\Bar{x}$ and in the computations we consider $x_i=x_1,x_j=x_2$ without loss of generality.

It is possible to directly extend the three-dimensional analysis of the previous subsection to this general case considering the model to be $y=w_1x_1+w_2x_2+wx+\epsilon$, with $w=1$ and $x=w_3x_3+...+w_Dx_D$. Recalling that in this case the third feature $x$ has general variance $\sigma^2_{x}$, the following lemma holds.

\begin{lemma}
Let $y=w_1x_1+...+w_Dx_D+\epsilon=w_1x_1+w_2x_2+wx+\epsilon$ with $\sigma^2_{x_1}=\sigma^2_{x_2}=1$ and $\sigma^2_x=\sigma^2_{w_3x_3+...+w_Dx_D}$. Then, performing linear regression in the asymptotic setting, the decrease of variance is greater than the increase of bias when aggregating the two features $x_1$ and $x_2$ with their average if and only if the condition on the correlation of Equation \eqref{eq:bound3D} holds (with the parameter $b$ expressed like in Equation \eqref{eq:general3D} as $b=\frac{\sigma_{x}(\rho_{x_1,x}-\rho_{x_2,x})w}{(w_1-w_2)}$).
\end{lemma}
\begin{proof}
The lemma follows by applying the three-dimensional analysis with general variance of the third feature $\sigma^2_x$ (Theorem \ref{thm:3Dbound} and Remark \ref{rem:gen3D}).
\end{proof}

\subsection{D-dimensional algorithm}\label{subs:NdimAlgo}

For the general $D$-dimensional case, as explained in the previous subsection, the three-dimensional results can be extended considering as third feature the linear combination of the $D-2$ features not currently considered for the aggregation. A drawback of applying the obtained result in practice is that it requires the knowledge of all the coefficients $w_1,...,w_D$, which is unrealistic, or to approximate them through an estimate, performing linear regression on the complete $D$-dimensional dataset. In this case, the computational cost is $\mathcal{O}(n\cdot D^2 + D^3)$---which becomes $\mathcal{O}(n\cdot D^2 + D^{2.37})$ if using the Coppersmith–Winograd algorithm~\citep{COPPERSMITH1990}---and it is impractical with a huge number of features. Therefore, since the equation in the three dimensional asymptotic analysis becomes equal to the bivariate one if the two features have the same correlation with the third (Remark \ref{rem:sameCorr}), it is reasonable, if they are highly correlated, to assume this to be valid and to apply the asymptotic bivariate result shown in Equation \eqref{eq:2DasymBuondRed} to decide whether the two features should be aggregated or not. 
In this way, we iteratively try all combinations of two features, with complexity $\mathcal{O}(n+D^2)$ in the worst case, in order to choose the groups of features that is convenient to aggregate with their mean.

\begin{algorithm}[ht]
\caption{\algnameshort: \algname}\label{alg:dimRed}
\begin{algorithmic}
\Require{$D$ features $\{x_1,\dots,x_D\}$; target $y$; $n$ samples}
\Ensure{reduced features $\{\bar{x}_1,\dots,\bar{x}_d\}$, with $d \leq D$}
\State{$\bm{\mathcal{P}} \leftarrow \{\}$}\Comment{Partition of the features}
\State{$\mathcal{V} \leftarrow \{\}$} \Comment{Set of already considered features}
\ForEach {$i \in \{ 1,\dots,D\}$}
\If{$i\not \in \mathcal{V}$}
    \State $\mathcal{P} \leftarrow \{i\}$
    \State $ \mathcal{V} \leftarrow  \mathcal{V} \cup \{i\}$  
    \ForEach{$j \in \{ i+1,\dots,D\}$}
    \State{$c \leftarrow \text{correlation}(x_i, x_j)$}
    \State{$b \leftarrow \text{threshold}(x_i, x_j, y)$}
    \If{$c \geq b$} \Comment{Aggregate the features}
    \State{$\mathcal{P} \leftarrow \mathcal{P}  \cup \{j\}$}
    \State{$\mathcal{V} \leftarrow  \mathcal{V} \cup \{j\}$  }
    \EndIf
    \EndFor
    \State{$\bm{\mathcal{P}} \leftarrow \bm{\mathcal{P}} \cup \{\mathcal{P}\}$}
\EndIf
\EndFor   
\State{$d \gets \lvert\bm{\mathcal{P}}\rvert$}
\ForEach{$k \in \{ 1,\dots,d\}$}
\State{$\bar{x}_k = \frac{1}{\lvert\mathcal{P}_k\rvert} \sum_{i \in \mathcal{P}_k} x_{i}$}
\EndFor\\
\Return{$\{\bar{x}_1,\dots,\bar{x}_d\}$}
\end{algorithmic}
\end{algorithm}

In Algorithm \ref{alg:dimRed} the pseudo-code of the proposed algorithm \emph{\algname} (\algnameshort) can be found. The proposed dimensionality reduction algorithm creates a $d$ dimensional partition of the indices of the features $\{1,\dots,D\}$ by iteratively comparing couples of features and adding them to the same subset if their correlation ($\text{correlation}(x_i,x_j)$) is greater than the threshold ($\text{threshold}(x_i,x_j,y)$), obtained from Equation \eqref{eq:2DasymBuondRed}. Then, it aggregates the features in each set $k$ of the partition ($\bm{\mathcal{P}}$) with their average, producing each output $\bar{x}_k$.

\section{Numerical Validation}\label{sec:appl}
In this section, the theoretical results obtained in Section \ref{sec:2Dalgorithm} and \ref{sec:3andD} are exploited to perform dimensionality reduction on synthetic datasets of two, three and $D$ dimensions. Furthermore, the proposed dimensionality reduction approach \emph{\algnameshort} is applied to four real datasets and compared with PCA and Supervised-PCA.\footnote{Code and datasets can be found at the following link: \url{https://github.com/PaoloBonettiPolimi/PaperLinCFA}.}
To evaluate the performance of the computed linear regressions, the results will be evaluated in terms of Mean Squared Error ($MSE$) and R-squared ($R^2$).

\subsection{Two-dimensional application}

\begin{table}[ht]
\sidewaystablefn%
\caption{Experiment on synthetic bivariate data for two combinations of weights and three different values of variance of the noise. 
\label{tab:bidimSmall}}
\centering 
\resizebox{\textwidth}{!}
{\begin{tabular}{@{}ccccccc @{}} \hline 
    & \multicolumn{3}{c}{95\% Confidence Interval ($w_1=0.2,w_2=0.8$)}& \multicolumn{3}{c}{95\% Confidence Interval ($w_1=0.47,w_2=0.52$ )}\\
\hline 
Quantity & $\sigma=0.5$ & $\sigma=1$ & $\sigma=10$ & $\sigma=0.5$ & $\sigma=1$ & $\sigma=10$ \\
\hline 
\# aggregations (theo) & $\mathbf{0}$ & $\mathbf{0}$ & $\mathbf{500}$ & $\mathbf{500}$ & $\mathbf{500}$ & $\mathbf{500}$ \\
\# aggregations (emp) & $\mathbf{0}$ & $\mathbf{24}$ & $\mathbf{332}$ & $\mathbf{314}$ & $\mathbf{339}$ & $\mathbf{346}$ \\
$R^2$ full & $0.781\pm5.2\mathrm{e}{-5}$ & $0.487\pm1.23\mathrm{e}{-4}$ & $0.010\pm2.64\mathrm{e}{-4}$ & $0.794\pm6.4\mathrm{e}{-5}$ & $0.505\pm9.1\mathrm{e}{-5}$ & $0.012\pm2.31\mathrm{e}{-4}$ \\
$R^2$ aggregate & $0.765\pm3.0\mathrm{e}{-5}$ & $0.486\pm7.8\mathrm{e}{-5}$ & $0.011\pm1.69\mathrm{e}{-4}$ & $0.794\pm6.0\mathrm{e}{-5}$ & $0.506\pm7.5\mathrm{e}{-5}$ & $0.015\pm2.0\mathrm{e}{-5}$ \\
\textit{MSE} full & $0.275\pm6.5\mathrm{e}{-5}$ & $1.000\pm2.40\mathrm{e}{-4}$ & $103.021\pm0.027$ & $0.263\pm8.2\mathrm{e}{-5}$ & $0.949\pm1.75\mathrm{e}{-4}$ & $94.573\pm2.209\mathrm{e}{-3}$\\
\textit{MSE} aggregate & $0.295\pm3.7\mathrm{e}{-5}$ & $1.004\pm1.52\mathrm{e}{-4}$ & $102.977\pm0.018$ & $0.262\pm7.7\mathrm{e}{-5}$ & $0.947\pm1.45\mathrm{e}{-4}$ & $94.363\pm1.914\mathrm{e}{-3}$\\
\hline 
\end{tabular}}
\end{table}

In the bivariate setting, according to Equation \eqref{eq:2DasymBuondRed} and \eqref{eq:2DfinBuondRed}, it is convenient to aggregate the two features with a small number of samples $n$, with a small absolute value of the difference between the coefficients of the linear model $w_1,w_2$ or with a large variance of the noise $\sigma^2$. The synthetic experiments (full description in Appendix \ref{app:exp}) confirm with data the theoretical result. In particular, they are performed with a fixed number of samples $n=500$, a fixed correlation between the features $\rho_{x_1,x_2}\approx 0.9$, comparing two combinations of weights (at small and large distances) and three different variances of the noise (small, normal, large).

Table \ref{tab:bidimSmall} shows the results of the experiments (more detailed results can be found in Table \ref{tab:2dimSyn},\ref{tab:2dimSynSmall} of Appendix \ref{app:exp}). In line with the theory, when the weights in the linear model are consistently distant, only with a huge variance of the noise the threshold is far from 1 and the two features are aggregated, while for a reasonably small amount of variance in the noise they are kept separated. On the other hand, when the weights in the linear model are similar, the threshold of Equation \eqref{eq:2DasymBuondRed} is small and the conclusion is to aggregate the two features also with a small amount of variance in the noise. 
The confidence intervals on the $R^2$ and on the \textit{MSE} confirm that, when the correlation is above the threshold, the performance of the linear model when the two features are aggregated with their average is statistically not worse than the bivariate model where they are kept separate. It is finally important to notice that, knowing the coefficients of the regression, always leads to aggregate the two features or not in all the $500$ repetitions of the experiment (row \emph{\# aggregations (theo)}). On the contrary, estimating the coefficients from data leads to the same action in most repetitions but not always (row \emph{\# aggregations (emp)}), since the limited amount of data introduces noise into the estimates.

\subsection{Three-dimensional application}

\begin{table}[h]
\caption{Synthetic experiment in the three dimensional setting comparing the full model with three variables with the bivariate model where $x_1,x_2$ are aggregated with their mean.
\label{tab:3dimSyn}}
\centering 
{\begin{tabular}{@{}cc@{}} \hline 
Quantity & 95\% Confidence Interval \\\hline 
\# Aggregations (theo) & $\mathbf{500}$ \\
\# Aggregations (emp) & $\mathbf{335}$ \\
$R^2$ full & $0.825\pm6\mathrm{e}{-6}$\\
$R^2$ aggregate & $0.825\pm5\mathrm{e}{-6}$\\
\textit{MSE} full & $0.285\pm9\mathrm{e}{-6}$\\
\textit{MSE} aggregate & $0.286\pm8\mathrm{e}{-6}$\\
\hline 
\end{tabular}}
\end{table}

Equation \eqref{eq:bound3D} expresses the interval for which it is convenient to aggregate the two features $x_1$and $x_2$ in the three-dimensional setting. As in the bivariate case, it is related to the number of samples, the difference between weights, and the variance of the noise. In addition, it also depends on the difference of the correlations between each of the two features with the third one $x_3$ and on the weight $w_3$.\\ 
The experiment performed in this setting is based on synthetic data, computed with the following realistic setting: the weights $w_1=0.4,\ w_2=0.6$ are closer than $w_3=0.2$. Moreover, the two features are significantly correlated: $\rho_{x_1,x_2}\approx0.88$ (more details can be found in Appendix \ref{app:exp}).

In this setting, as shown in Table \ref{tab:3dimSyn}, it is convenient to aggregate the two features $x_1,x_2$ with their average both in terms of \textit{MSE} and $R^2$, since the aggregation does not worsen the performances. In particular, the aggregation is already convenient with a small standard deviation of the noise ($\sigma=0.5$).

\subsection{$D$-dimensional application}

\begin{table}[th]
\caption{Synthetic experiment in the $D$ dimensional setting. The experiment has been repeated twice: considering the theoretical threshold with the exact coefficients (\textit{theo}) and with coefficients estimated from data (\textit{emp}). \label{tab:NdimSyn}}
\centering 
{\begin{tabular}{@{}cc@{}} \hline 
Quantity & 95\% Confidence Interval \\\hline 
$R^2$ full & $0.828\pm1.46\mathrm{e}{-4}$\\
$R^2$ aggregate (theo) & $0.890\pm4.8\mathrm{e}{-5}$\\
$R^2$ aggregate (emp) & $0.881\pm1.07\mathrm{e}{-4}$\\
\textit{MSE} full & $157.346\pm0.120$\\
\textit{MSE} aggregate (theo) & $100.536\pm0.040$\\
\textit{MSE} aggregate (emp) & $108.725\pm0.088$\\
\hline 
Number of reduced variables (theo) & $\mathbf{4}$ \\
Number of reduced variables (emp) & $\mathbf{15}$ \\
\hline
\end{tabular}}
\end{table}

This subsection introduces the $D$-dimensional synthetic experiment performed $500$ times with $n=500$ samples and $D=100$ features, reduced with the proposed algorithm \emph{\algnameshort} (more details can be found in Appendix \ref{app:exp}).

The test results shown in Table \ref{tab:NdimSyn} underline that knowing the real values of the coefficients of the linear model would lead to a reduced dataset of $d=4$ features and a significant increase of performance (\emph{$R^2$ aggregate (theo), $MSE$ aggregate (theo)}), while using the empirical coefficients the dimension is reduced to $d=15$, still with a significant increase of performance both in terms of $ \text{MSE}$ and $R^2$ (\emph{$R^2$ aggregate (emp), $MSE$ aggregate (emp)}). This is a satisfactory result and it is confirmed by the real dataset application described below.

\begin{figure*}[h]
\centering
\subfloat[]{\includegraphics[width=2.3in]{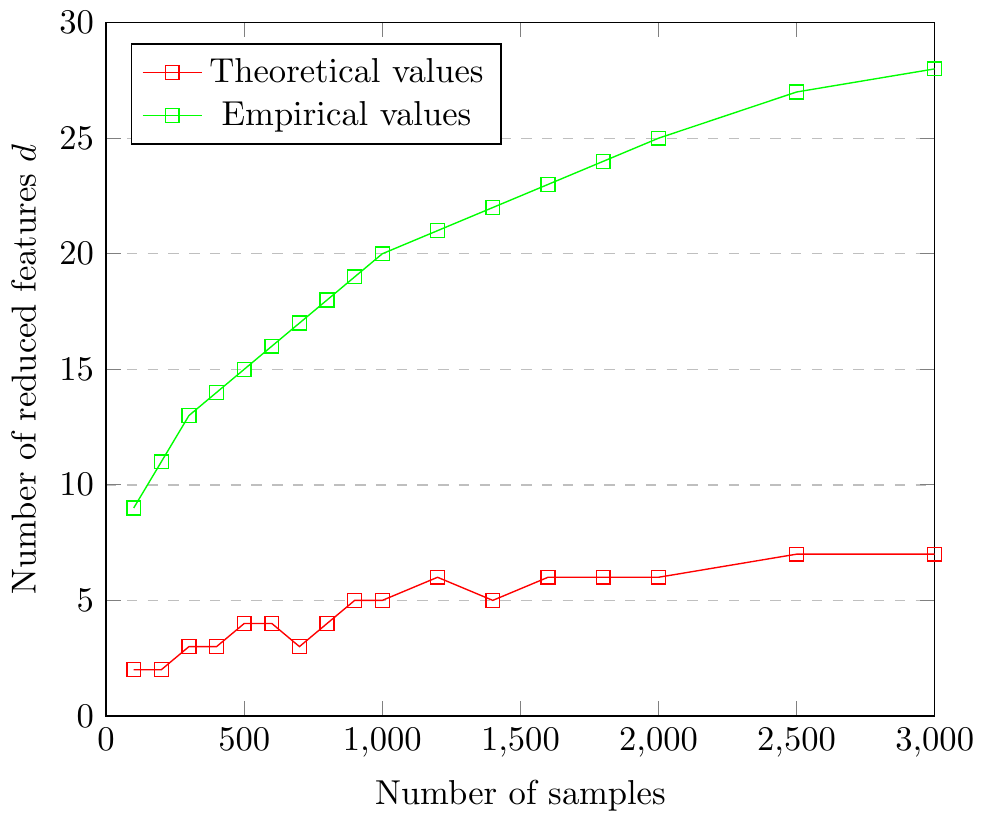}%
\label{fig:nFeatures}}
\hfil
\subfloat[]{\includegraphics[width=2.3in]{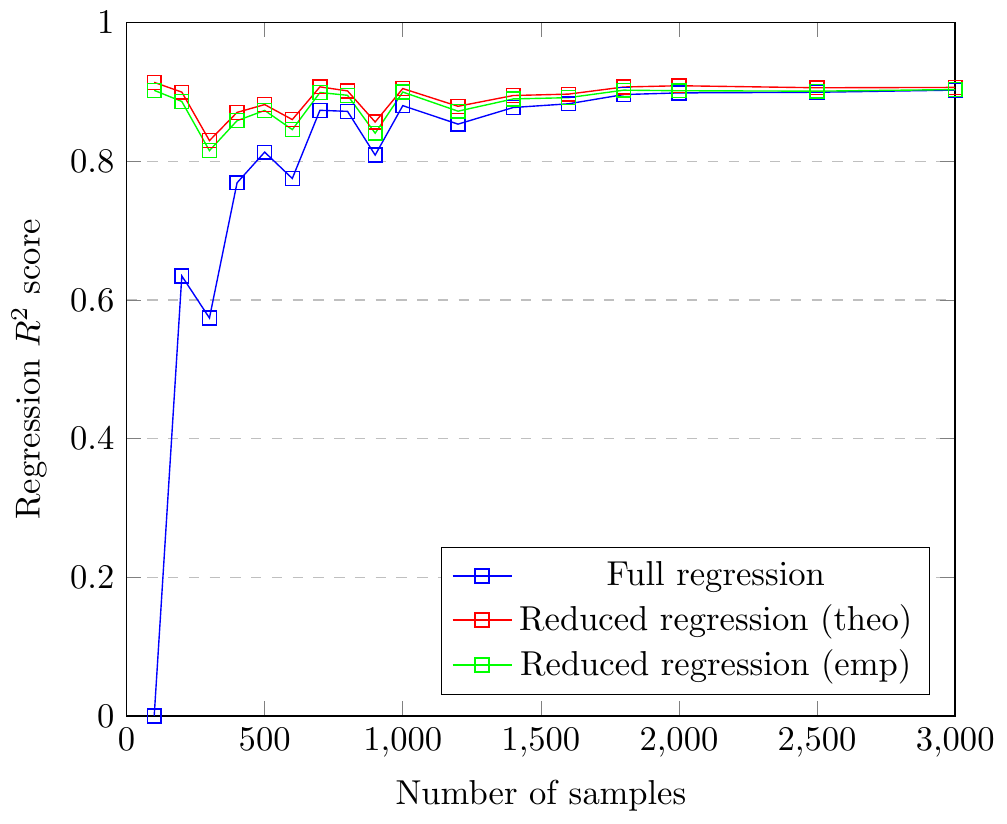}%
\label{fig:R2}}
\hfil
\subfloat[]{\includegraphics[width=2.3in]{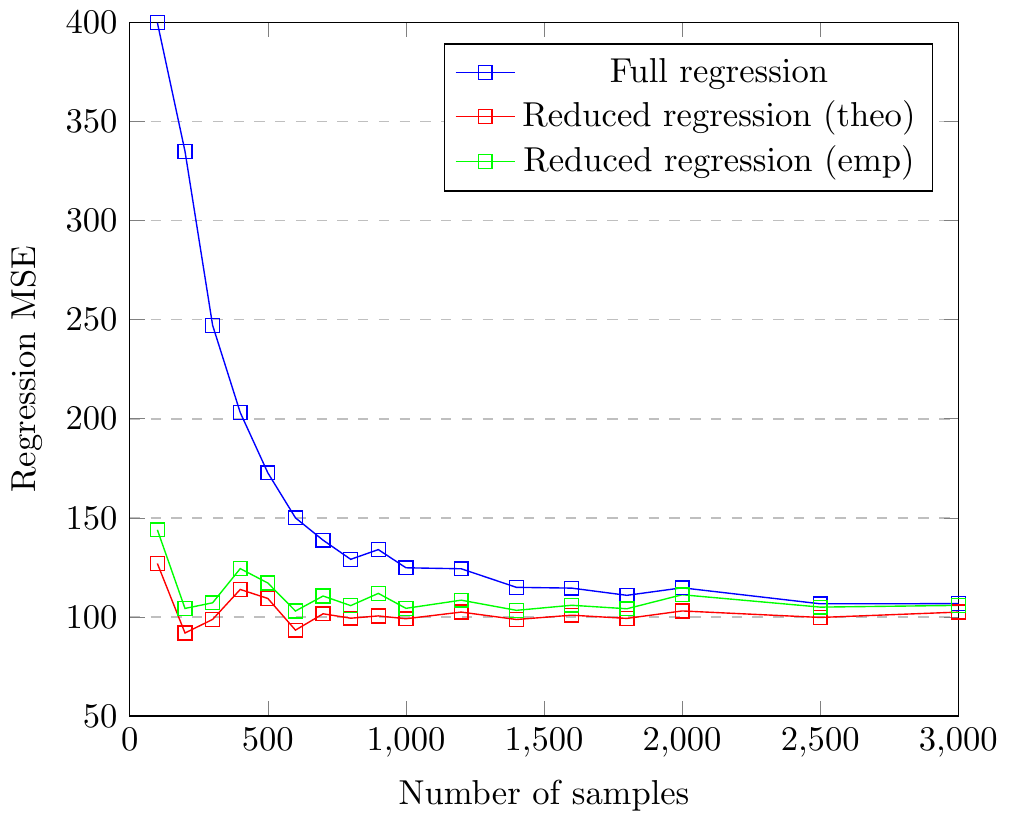}%
\label{fig:MSE}}
\caption{Figure \ref{fig:nFeatures} shows the number of reduced features for a different number of samples. Figure \ref{fig:R2},\ref{fig:MSE} show the regression performance in terms of $R^2$ and \textit{MSE} for different number of samples. Blue lines refer to the linear regression with all the original features, while red and green lines respectively refer to linear regression on the features reduced by applying the proposed algorithm considering theoretical and empirical quantities.}
\label{fig:nSamples}
\end{figure*}

To better understand the performance of the algorithm, in Figure \ref{fig:nSamples} we consider the number of selected features and the regression scores. From Figure \ref{fig:nFeatures} it is clear that with a small number of samples, both considering theoretical and empirical quantities, the number of reduced features $d$ becomes smaller to prevent overfitting. Moreover, considering the empirical quantities, which are the only ones available in practice, lead to a larger number of reduced features (but still significantly smaller than the original dimension $D$). Figure \ref{fig:R2},\ref{fig:MSE} show the performance of the linear regression considering the reduced features compared with the full dataset. When the number of samples is significantly larger than the number of features, the performance of the reduced datasets is only slightly better but, when the number of samples is of the same order of magnitude as the number of features, the reduced datasets (both considering empirical and theoretical quantities) significantly outperform the regression over the full dataset. Moreover, the regression performed with reduced datasets is much robust, since it has a score that is stable for different numbers of samples.

The main practical result introduced in this paper (the algorithm \textit{\algnameshort}) has been also tested on four real datasets. In particular, the results of the application of the dimensionality reduction method introduced in this paper are discussed in comparison with the chosen baselines: PCA and Supervised PCA\footnote{The code for Supervised PCA has been adapted starting from the implementation provided at the following link:\\
\url{https://github.com/kumarnikhil936/supervised_pca}}.

\begin{table}[]
\caption{Experiments on real datasets. The total number of samples $n$ has been divided into train (66\% of data) and test (33\% of data) sets.
\label{tab:realRes}}
\centering 
{\begin{tabular}{@{}ccccc@{}} \hline 
Quantity & Life Exp & Financial & Climatological I & Climatological II \\\hline 
\# samples $n$ & $1649$ & $1299$ & $1038$ & $981$\\
\hline
Full dim (\# features $D$) & $18$ & $75$ & $136$ & $1991$\\
Reduced dim PCA  & $13$ & $29$ & $13$ & $20$\\
Reduced dim supervised PCA  & $4$ & $4$ & $43$ & $40$\\
Reduced dim \algnameshort & $14$ & $12$ & $38$ & $37$\\
\hline
$R^2$ full & $0.834$ & $-1.441$ & $0.298$ & $-1.402$\\
$R^2$ PCA & $0.831$ & $0.846$ & $0.486$ & $0.775$\\
$R^2$ supervised PCA & $0.809$ & $0.863$ & $0.585$ & $0.867$\\
$R^2$ \algnameshort & $0.835$ & $0.885$ & $0.604$ & $0.922$\\
\hline
\textit{MSE} full & $0.180$ & $3.702$ & $0.286$ & $1.277$\\
\textit{MSE} PCA & $0.192$ & $0.234$ & $0.209$ & $0.205$\\
\textit{MSE} supervised PCA & $0.207$ & $0.208$ & $0.168$ & $0.120$\\
\textit{MSE} \algnameshort & $0.179$ & $0.162$ & $0.145$ & $0.070$\\
\hline 
\end{tabular}}
\end{table}

Specifically, the algorithm has been applied to a dataset of $18$ features predicting life expectancy and three large dimensional datasets related to finance and climate (the datasets are described in more detail in Appendix \ref{app:exp}). Table \ref{tab:realRes} shows the $MSE$ and $R^2$ coefficients obtained with linear regression applied on the full dataset, on the dataset reduced by \algnameshort, and on the dataset reduced by PCA and Supervised-PCA. The number of components selected for PCA is set to explain $95\%$ of variance, while for supervised PCA it is reported the best result (evaluating from $d=1$ to $d=50$ principal components).

It is possible to notice that in the first case, when the number of features is low, the results are similar between the full regression and the regression on the reduced dataset applying supervised PCA, PCA or \emph{\algnameshort}. On the other hand, when the algorithms are applied to the large dimensional data, the algorithm that we propose obtains slightly better performances than the other two methods. Therefore, the main result of this paper is able, in linear settings, to reduce the dimensionality of the input features improving (or not-worsening) the performance of the linear model and, most importantly, preserving the interpretability of the reduced features.

\section{Conclusion and Future Work}\label{sec:conclusion}
This paper presents a dimensionality reduction algorithm in linear settings with the theoretical guarantee to produce a reduced dataset that does not perform worse than the full dataset in terms of $MSE$, with a decrease of variance that is larger than the increase of bias due to the aggregation of features. The main strength of the proposed approach is that it aggregates features through their mean, which reduces the dimension meanwhile preserving the interpretability of each feature, which is not common in traditional dimensionality reduction approaches like PCA. Moreover, the complexity of the proposed algorithm is lower than performing a linear regression on the full original dataset. The main weaknesses of the proposed method are that all the computations have been done assuming the features to be continuous and the relationship between the target and the features to be linear, which is a strong assumption in real-world applications. However, the empirical results show an increase in performance and a significant reduction of dimensionality when applying the proposed algorithm to real-world datasets on which linear regression has good performances.

In future work it may be interesting to relax the linearity assumption, considering the target as a general function of the input features and applying a general machine learning method to the data. Another possible way to enrich the results obtained in this paper is to consider structured data, where prior knowledge of their relationship can be useful to identify the most suitable features for the aggregation (e.g., on climatological data, features that are registered by two adjacent sensors are more likely to be aggregated).

\section*{Declarations and Acknowledgments}
This work has been supported by the CLINT research project funded by the H2020 Programme of the European Union under Grant Agreement No 101003876.\\
The authors have no competing interests to declare that are relevant to the content of this article.

\setlength{\bibsep}{0pt plus 0.3ex}
\bibliography{ms}

\newpage

\begin{appendices}

\section{Two-dimensional analysis: additional proofs and results}\label{app:twoDim}
This section shows proofs and additional technical results that are not reported in Section \ref{sec:2Dalgorithm} to keep the exposition clear.

\subsection{Variance}
This subsection contains some additional proofs related to the bivariate analysis of variance presented in the main paper.
\begin{proof}[Proof of Equation \eqref{eq:VarEst}]\label{proof:varEstim}
Given the training set of features $\mathbf{X}$ and target $\mathbf{y}$, in a linear regression model, the estimated weights are computed as $\hat{w} = (\mathbf{X}^T\mathbf{X})^{-1}\mathbf{X}^T\mathbf{y}$. Therefore: 
\begin{equation*}
\begin{gathered}
var_{\mathcal{T}}(\hat{w}\lvert\mathbf{X}) = var_{\mathcal{T}}((\mathbf{X}^T\mathbf{X})^{-1}\mathbf{X}^T \mathbf{y}\lvert\mathbf{X})\\
= (\mathbf{X}^T\mathbf{X})^{-1}\mathbf{X}^T\mathbf{X}(\mathbf{X}^T\mathbf{X})^{-1}var_{\mathcal{T}}(\mathbf{y}\lvert\mathbf{X}). 
\end{gathered}
\end{equation*}
Since $(\mathbf{X}^T\mathbf{X})^{-1}\mathbf{X}^T\mathbf{X}=\mathbf{I}$ and $var_{\mathcal{T}}(\mathbf{y}\lvert\mathbf{X}) = \sigma^2$ by hypothesis, the result follows.
\end{proof}
\begin{proof}[Proof of Lemma \ref{lemma:varianceEstimator}]
To prove this results it is enough to start from Equation \eqref{eq:VarEst} and substitute the values of $\mathbf{X}$.\\
For the one-dimensional setting: 
\begin{equation*}
\begin{gathered}
\small
var_{\mathcal{T}}(\hat{w}\lvert\mathbf{X}) = (\mathbf{X}^T\mathbf{X})^{-1}\sigma^2\\
=\Big(\begin{bmatrix} \bar{x}^1&...&\bar{x}^n \end{bmatrix} \begin{bmatrix} \bar{x}^1 \\ ... \\ \bar{x}^n \end{bmatrix}\Big)^{-1}\sigma^2= \frac{\sigma^2}{\sum_{i=1}^n (\bar{x}^i)^2}.
\end{gathered}
\end{equation*}
Recalling that the expected value of the random variables $x_1$ and $x_2$ is zero by hypothesis, then $\sum_{i=1}^n (\bar{x}^i)^2=(n-1)\hat{\sigma}^2_{\bar{x}}.$
\\
For the two dimensional setting: 
\begin{equation*}
\begin{gathered}
var_{\mathcal{T}}(\hat{w}\lvert\mathbf{X}) = (\mathbf{X}^T\mathbf{X})^{-1}\sigma^2\\
= \Bigg(\begin{bmatrix} x^1_1 & ... & x^n_1 \\ x_2^1 & ... & x_2^n \end{bmatrix} \begin{bmatrix} x^1_1 & x^1_2 \\ ... & ... \\ x_1^n & x_2^n \end{bmatrix}\Bigg)^{-1}\sigma^2 
\\= \Bigg(\begin{bmatrix} (x^1_1)^2+...+(x^n_1)^2 & x^1_1x^1_2+...+x^n_1x^n_2 \\ x^1_1x^1_2+...+x^n_1x^n_2 & (x^1_2)^2+...+(x^n_2)^2 \end{bmatrix}\Bigg)^{-1}\sigma^2\\
=\frac{\sigma^2}{(\sum_{i=1}^n(x_1^i)^2\sum_{i=1}^n(x_2^i)^2)-(\sum_{i=1}^n(x_1^ix_2^i))^2}\\
\times \begin{bmatrix} (x^1_2)^2+...+(x^n_2)^2 & -(x^1_1x^1_2+...+x^n_1x^n_2) \\ -(x^1_1x^1_2+...+x^n_1x^n_2) & (x^1_1)^2+...+(x^n_1)^2 \end{bmatrix}.
\end{gathered}
\end{equation*}
The result follows recalling again that the expected value of the random variables $x_1,x_2$ is zero, therefore $\sum_{i=1}^n(x_1^i)^2=(n-1)\hat{\sigma}^2_{x_1}$, $\sum_{i=1}^n(x_2^i)^2=(n-1)\hat{\sigma}^2_{x_2}$ and $\sum_{i=1}^n(x_1^ix_2^i)=(n-1)\hat{cov}(x_1,x_2)$.
\end{proof}

\begin{proof}[Proof of Theorem \ref{thm:variance2D}]
Let us consider a training dataset $\mathcal{T}$ and a univariate test sample $(x,y)$. Then the variance is:
\begin{equation*}
\E_{x,\mathcal{T}}[(h_\mathcal{T}(x)-\Bar{h}(x))^2]=\E_{x}\E_\mathcal{T}[(h_\mathcal{T}(x)-\Bar{h}(x))^2].\end{equation*}
Therefore, for the one-dimensional regression:
\begin{equation*}
\begin{gathered}
\E_{x}\E_\mathcal{T}[(h_\mathcal{T}(x)-\Bar{h}(x))^2]= \E_{x}\E_\mathcal{T}[(\hat{w}x-\E_{\mathcal{T}}[\hat{w} x])^2]\\
=\E_{x}\E_\mathcal{T}[(x(\hat{w}-\E_\mathcal{T}[\hat{w}])^2]=\E_{x}[x^2]\E_{\mathcal{T}}[(\hat{w}-\E_\mathcal{T}[\hat{w}])^2]\\
= var_x(x)var_\mathcal{T}(\hat{w})=\sigma^2_x var_\mathcal{T}(\hat{w}).
\end{gathered}
\end{equation*}
Conditioning on the features training set $\mathbf{X}$:
\begin{equation*}
\E_{x}\E_\mathcal{T}[(h_\mathcal{T}(x)-\Bar{h}(x))^2 \lvert \mathbf{X}] = \sigma^2_x var_\mathcal{T}(\hat{w}\lvert\mathbf{X}) = \frac{\sigma_x^2\sigma^2}{(n-1)\hat{\sigma}^2_x}.
\end{equation*}
Regarding the two dimensional regression: 
\begin{equation*}
\begin{gathered}
\E_{x}\E_\mathcal{T}[(h_\mathcal{T}(x)-\Bar{h}(x))^2]\\
=\E_{x}\E_\mathcal{T}[(\hat{w}_1x_1+\hat{w}_2x_2-\E_{\mathcal{T}}[\hat{w}_1x_1+\hat{w}_2x_2])^2]\\
=\E_{x}\E_\mathcal{T}[(x_1(\hat{w}_1-\E_{\mathcal{T}}[\hat{w}_1])+x_2(\hat{w}_2-\E_{\mathcal{T}}[\hat{w}_2]))^2]\\
\medmuskip=0mu
\thinmuskip=0mu
\thickmuskip=0mu
=\E_{x}\E_\mathcal{T}[(x_1(\hat{w}_1-\E_{\mathcal{T}}[\hat{w}_1]))^2]+\E_{x}\E_\mathcal{T}[(x_2(\hat{w}_2-\E_{\mathcal{T}}[\hat{w}_2]))^2]\\ 
+2\E_{x}\E_\mathcal{T}[x_1x_2(\hat{w}_1-\E_{\mathcal{T}}[\hat{w}_1])(\hat{w}_2-\E_{\mathcal{T}}[\hat{w}_2])]\\  =var_x(x_1)var_\mathcal{T}(\hat{w}_1)+var_x(x_2)var_\mathcal{T}(\hat{w}_2)\\
+2cov_x(x_1,x_2)cov_\mathcal{T}(\hat{w}_1,\hat{w}_2). 
\end{gathered}
\end{equation*}

Conditioning on the features training set:
\begin{equation*}
\begin{gathered}
\E_{x}\E_\mathcal{T}[(h_\mathcal{T}(x)-\Bar{h}(x))^2 \lvert \mathbf{X}]\\ =var_x(x_1)var_\mathcal{T}(\hat{w}_1\lvert\mathbf{X}) + var_x(x_2)var_\mathcal{T}(\hat{w}_2\lvert\mathbf{X})\\
+2cov_x(x_1,x_2)cov_\mathcal{T}(\hat{w}_1,\hat{w}_2\lvert\mathbf{X})\\
=\frac{\sigma^2(\sigma^2_{x_1}\hat{\sigma}^2_{x_2}+\sigma^2_{x_2}\hat{\sigma}^2_{x_1}-2cov(x_1,x_2)\hat{cov}(x_1,x_2))}{(n-1)(\hat{\sigma}^2_{x_1}\hat{\sigma}^2_{x_2} - \hat{cov}(x_1,x_2)^2)}. 
\end{gathered}
\end{equation*}
\end{proof}

\begin{proof}[Proof of Theorem \ref{thm:diffFinite}]
From Theorem \ref{thm:variance2D}, the difference of variances between the two-dimensional and the one-dimensional cases is:
\begin{equation*}
\begin{gathered}
\frac{\sigma^2(\sigma^2_{x_1}\hat{\sigma}^2_{x_2}+\sigma^2_{x_2}\hat{\sigma}^2_{x_1}-2cov(x_1,x_2)\hat{cov}(x_1,x_2))}{(n-1)(\hat{\sigma}^2_{x_1}\hat{\sigma}^2_{x_2} - \hat{cov}(x_1,x_2)^2)}\\
-\sigma_{x_1+x_2}^2\frac{\sigma^2}{(n-1)\hat{\sigma}^2_{x_1+x_2}},
\end{gathered}
\end{equation*}

\noindent that with the assumptions of Equation \eqref{eq:sameVar} can be written as:
\begin{equation*}
\begin{gathered}
\frac{\sigma^2(2\sigma^2_{x}\hat{\sigma}^2_{x}-2cov(x_1,x_2)\hat{cov}(x_1,x_2))}{(n-1)(\hat{\sigma}^4_{x} - \hat{cov}(x_1,x_2)^2)}\\
-\sigma_{x_1+x_2}^2\frac{\sigma^2}{(n-1)\hat{\sigma}^2_{x_1+x_2}}.
\end{gathered}
\end{equation*}
Recalling that $\sigma^2_{x_1+x_2}=\sigma^2_{x_1}+\sigma^2_{x_2}+2cov(x_1,x_2)$, and that the same applies for the sample variance, the expression above is equal to:
\begin{equation*}
\begin{gathered}
\frac{\sigma^2(2\sigma^2_{x}\hat{\sigma}^2_{x}-2cov(x_1,x_2)\hat{cov}(x_1,x_2))}{(n-1)(\hat{\sigma}^4_{x} - \hat{cov}(x_1,x_2)^2)}\\
-(2\sigma^2_x+2cov(x_1,x_2))\frac{\sigma^2}{(n-1)(2\hat{\sigma}^2_{x}+2\hat{cov}(x_1,x_2))}.
\end{gathered}
\end{equation*}
Applying the common denominator the result follows:
\begin{equation*}
\begin{gathered}
\frac{\sigma^2}{(n-1)(\hat{\sigma}^4_{x} - \hat{cov}(x_1,x_2)^2)}\\
\times[(2\sigma^2_{x}\hat{\sigma}^2_{x}-2cov(x_1,x_2)\hat{cov}(x_1,x_2))\\
-(\sigma^2_x+cov(x_1,x_2))(\hat{\sigma}^2_x-\hat{cov}(x_1,x_2))]\\
= \frac{\sigma^2(\sigma^2_x-cov(x_1,x_2))(\hat{\sigma}^2_x+\hat{cov}(x_1,x_2))}{(n-1)\hat{\sigma}^4_{x}(1- \hat{\rho}_{x_1,x_2}^2)}\\
=\frac{\sigma^2}{(n-1)}\cdot\frac{\sigma^2_x(1-\rho_{x_1,x_2})}{\hat{\sigma}^2_x(1-\hat{\rho}_{x_1,x_2})}.
\end{gathered}
\end{equation*}
\end{proof}

\subsection{Bias}
 This subsection contains some technical results and proofs used to compute the difference of biases between the two considered model in the bivariate linear setting of the main paper.
\subsubsection{Expected value of the estimators}\label{subsubsec:expvalEst}
The expected value with respect to the training set $\mathcal{T}$ of the vector $\hat{w}$ of the regression coefficients estimates is necessary for the computations of the bias of the models. Given the training features $\mathbf{X}$, its known expression in a general problem $y=f(\mathbf{X})+\epsilon$ is given by~\citep{johnson2007}:
\begin{equation}\label{eq:ExpvalEst}
\E_{\mathcal{T}}[\hat{w}\lvert\mathbf{X}] = (\mathbf{X}^T\mathbf{X})^{-1}\mathbf{X}^T f(\mathbf{X}).
\end{equation}

\begin{proof}
Given the training set of features $\mathbf{X}$ and target $\mathbf{y}$, in a linear regression model, the estimated weights are computed as $\hat{w} = (\mathbf{X}^T\mathbf{X})^{-1}\mathbf{X}^T\mathbf{y}$. Therefore: 
\begin{equation*}
\begin{gathered}
\E_{\mathcal{T}}[\hat{w}\lvert\mathbf{X}] = \E_{\mathcal{T}}[(\mathbf{X}^T\mathbf{X})^{-1}\mathbf{X}^T \mathbf{y}\lvert\mathbf{X}]\\ =\E_{\mathcal{T}}[(\mathbf{X}^T\mathbf{X})^{-1}\mathbf{X}^T(f(\mathbf{X})+\epsilon)\lvert\mathbf{X}]\\
=(\mathbf{X}^T\mathbf{X})^{-1}\mathbf{X}^T\E_{\mathcal{T}}[f(\mathbf{X})\lvert\mathbf{X}]+(\mathbf{X}^T\mathbf{X})^{-1}\mathbf{X}^T\E_{\mathcal{T}}[\epsilon\lvert\mathbf{X}]\\
=(\mathbf{X}^T\mathbf{X})^{-1}\mathbf{X}^T f(\mathbf{X}),
\end{gathered}
\end{equation*}
where the last equality holds since the expected value of the noise term $\epsilon$ is null by hypothesis.
\end{proof}

The following lemma shows the expected value of the weights for the two models that we are considering.
\begin{lemma}\label{lemma:expvalEstimator}
Let the real model be linear with respect to the features $x_1$ and $x_2$ ($y=w_1x_1+w_2x_2+\epsilon$). Then, in the one-dimensional case $\hat{y}=\hat{w}\bar{x}$, we have:
\begin{equation}
\medmuskip=0mu
\thinmuskip=0mu
\thickmuskip=0mu
 \E_{\mathcal{T}}[\hat{w}\lvert\mathbf{X}] =\frac{2(w_1\hat{\sigma}^2_{x_1}+w_2\hat{\sigma}^2_{x_2}+(w_1+w_2)\hat{cov}(x_1,x_2))}{\hat{\sigma}^2_{x_1}+\hat{\sigma}^2_{x_2}+2\hat{cov}(x_1,x_2)}.
\label{eq:expvalEst1d}
\end{equation}
In the two-dimensional case $\hat{y}=\hat{w}_1x_1+\hat{w}_2x_2$ the estimators are unbiased: 
\begin{equation}
\E_{\mathcal{T}}[\hat{w}\lvert\mathbf{X}] = \begin{bmatrix} w_1 \\ w_2 \end{bmatrix}.
\label{eq:expvalEst2d}
\end{equation}
\end{lemma}
\begin{proof}
To prove this result it is enough to apply Equation \eqref{eq:ExpvalEst} in the two settings.\\
In the one dimensional case, with $x=\frac{x_1+x_2}{2}$, it becomes:
\begin{equation*}
\begin{gathered}
\E_{\mathcal{T}}[\hat{w}\lvert\mathbf{X}] = \frac{\sum_{i=1}^n x^if(x)}{(n-1)\hat{\sigma}^2_x}\\
=\frac{2(\sum_{i=1}^n x_1^if(x_1^i,x_2^i)+\sum_{i=1}^n x_2^if(x_1^i,x_2^i))}{(n-1)\hat{\sigma}^2_{x_1+x_2}}.
\end{gathered}
\end{equation*}
Assuming the real model to be linear ($y=w_1x_1+w_2x_2+\epsilon$):
\begin{equation*}
\begin{gathered}
\E_{\mathcal{T}}[\hat{w}\lvert\mathbf{X}] =\frac{2(\sum_{i=1}^n x_1^if(x_1^i,x_2^i)+\sum_{i=1}^n x_2^if(x_1^i,x_2^i))}{(n-1)\hat{\sigma}^2_{x_1+x_2}}\\
= \frac{2(\sum_{i=1}^n x_1^i(w_1x_1^i+w_2x_2^i)+\sum_{i=1}^n x_2^i(w_1x_1^i+w_2x_2^i))}{(n-1)\hat{\sigma}^2_{x_1+x_2}}\\
\medmuskip=0mu
\thinmuskip=0mu
\thickmuskip=0mu
=\frac{2(w_1\sum_{i=1}^n(x_1^i)^2+(w_1+w_2)\sum_{i=1}^n x_1^ix_2^i+w_2\sum_{i=1}^n (x_2^i)^2)}{(n-1)\hat{\sigma}^2_{x_1+x_2}}.
\end{gathered}
\end{equation*}
Remembering that the expected values of $x_1,x_2$ are equal to zero it is possible to substitute the summations with sample variances and covariances, obtaining the result.\\
In the two dimensional setting, from the general result and substituting $f(\mathbf{X})=\mathbf{X}w$:
\begin{equation*}
\E_{\mathcal{T}}[\hat{w}\lvert\mathbf{X}] = (\mathbf{X}^T\mathbf{X})^{-1}\mathbf{X}^T f(\mathbf{X}) = (\mathbf{X}^T\mathbf{X})^{-1}\mathbf{X}^T \mathbf{X}w = w.
\end{equation*}
\end{proof}

\subsubsection{Bias of the model}\label{subsubsec:biasMod}
In Equation \eqref{eq:BiasVarDec}, we defined the (squared) bias as follows:
\begin{equation}
\E_x[(\Bar{h}(x)-\bar{y})^2]= \E_x[(\E_{\mathcal{T}}[h(x)]-\E_{y\lvert x}[y])^2].
\label{eq:modelBias2}
\end{equation}
The following result shows the bias of the two specific models considered in this section.

\begin{theorem}\label{thm:bias2D}
Let the real model be linear with respect to the two features $x_1,x_2$ ($y=w_1x_1+w_2x_2+\epsilon$). Then, in the one dimensional case $y=\hat{w}\bar{x}$, we have: 
\begin{equation}
\label{eq:bias1D}
\begin{aligned}
& \E_x[(\Bar{h}(x)-\bar{y})^2] \\
\medmuskip=0mu
\thinmuskip=0mu
\thickmuskip=0mu
& =\frac{\sigma^2_{x_1+x_2}}{(\hat{\sigma}^2_{x_1+x_2})^2}(w_1\hat{\sigma}^2_{x_1}+w_2\hat{\sigma}^2_{x_2}+(w_1+w_2)\hat{cov}(x_1,x_2))^2\\
& \quad +w_1^2\sigma^2_{x_1}+w_2^2\sigma^2_{x_2}+2w_1w_2cov(x_1,x_2)\\
& \quad -\frac{2}{\hat{\sigma}^2_{x_1+x_2}}(w_1\sigma^2_{x_1}+w_2\sigma^2_{x_2}+(w_1+w_2)cov(x_1,x_2))\\
& \quad \times (w_1\hat{\sigma}^2_{x_1}+w_2\hat{\sigma}^2_{x_2}+(w_1+w_2)\hat{cov}(x_1,x_2)).
\end{aligned}
\end{equation}

On the other hand, in the two dimensional case $y=\hat{w}_1x_1+\hat{w}_2x_2$ the model is unbiased:
\begin{equation}
\label{eq:bias2D}
\E_x[(\Bar{h}(x)-\bar{y})^2] = 0.
\end{equation}

\end{theorem}

\begin{proof}
The proof combines the results of Lemma \ref{lemma:expvalEstimator} with the definition of bias given in Equation \eqref{eq:modelBias2}.\\
Let us consider a training dataset $\mathcal{T}$ and a test sample $x,y$. Given the definition of (squared) bias:
\begin{equation*}
\E_x[(\Bar{h}(x)-\bar{y})^2]= \E_x[(\E_\mathcal{T}[h(x)]-\E_{y\lvert x}[y])^2],
\end{equation*}
in the one dimensional case, considering $x=\frac{x_1+x_2}{2}$: 
\begin{equation*}
\begin{gathered}
\E_{x}[(\E_{\mathcal{T}}[\hat{w} x]-(w_1x_1+w_2x_2))^2]\\
=\E_{x}[(x\E_{\mathcal{T}}[\hat{w}]-(w_1x_1+w_2x_2))^2]\\
=\E_{x}[x^2\E_{\mathcal{T}}[\hat{w}]^2]+\E_{x}[(w_1x_1+w_2x_2)^2]\\
-2\E_{x}[\E_{\mathcal{T}}[\hat{w}]x(w_1x_1+w_2x_2)].
\end{gathered}
\end{equation*}
Conditioning on the features training set and exploiting the independence between train and test set:
\begin{equation*}
\begin{gathered}
\E_x[(\Bar{h}(x)-\bar{y})^2\lvert \mathbf{X}] \\ =\sigma^2_x\E_{\mathcal{T}}[\hat{w}\lvert\mathbf{X}]^2+\E_{x}[(w_1x_1+w_2x_2)^2]\\
-2\E_{\mathcal{T}}[\hat{w}\lvert \mathbf{X}]\E_x[x(w_1x_1+w_2x_2)].
\end{gathered}
\end{equation*}

That, substituting $x=\frac{x_1+x_2}{2}$, is equal to:
\begin{equation*}
\begin{gathered}
\frac{1}{4}(\sigma^2_{x_1}+\sigma^2_{x_2}+2cov(x_1,x_2))\E_{\mathcal{T}}[\hat{w}\lvert\mathbf{X}]^2\\
+(w_1^2\sigma^2_{x_1}+w_2^2\sigma^2_{x_2}+2w_1w_2cov(x_1,x_2))\\
-\E_{\mathcal{T}}[\hat{w}\lvert\mathbf{X}](w_1\sigma^2_{x_1}+w_2\sigma^2_{x_2}+(w_1+w_2)cov(x_1,x_2)).
\end{gathered}
\end{equation*}
Substituting in the last equation the expression found in Lemma \ref{lemma:expvalEstimator} for $\E_{\mathcal{T}}[\hat{w}\lvert\mathbf{X}]$ in the one-dimensional setting, the result follows.
\\ \ \\
In the two dimensional regression, the bias is: 
$$\E_x[(\Bar{h}(x)-\bar{y})^2] =
\E_{x}[(\E_{\mathcal{T}}[\hat{w_1}x_1+\hat{w_2}x_2]-(w_1x_1+w_2x_2))^2].$$
Conditioning on the features training set, exploiting the independence between train and test set and recalling $\E_\mathcal{T}[\hat{w}\lvert\mathbf{X}] = \begin{bmatrix} w_1 \\ w_2 \end{bmatrix}$:
\begin{equation*}
\begin{gathered}
\E_{x}[(x_1\E_{\mathcal{T}}[\hat{w_1}]+x_2\E_{\mathcal{T}}[\hat{w_2}]-(w_1x_1+w_2x_2))^2]\\ =\E_{x}[(x_1w_1+x_2w_2-w_1x_1-w_2x_2)^2]=0.
\end{gathered}
\end{equation*}
\end{proof}

\subsubsection{Comparisons}
The asymptotic and finite-samples results for the comparisons of biases can be found in the main paper, in this subsection the related proofs are shown.
\begin{proof}
[Proof of Theorem \ref{thm:asym2dBias}]
Considering the bias of the two models computed in Theorem \ref{thm:bias2D} and exploiting consistency of the estimators, the difference between the one and the two dimensional model is equal to:
\begin{equation*}
\begin{gathered}
\frac{\sigma^2_{x_1+x_2}}{(\sigma^2_{x_1+x_2})^2}
(w_1\sigma^2_{x_1}+w_2\sigma^2_{x_2}+(w_1+w_2)cov(x_1,x_2))^2\\
+w_1^2\sigma^2_{x_1}+w_2^2\sigma^2_{x_2}+2w_1w_2cov(x_1,x_2)\\
-\frac{2}{\sigma^2_{x_1+x_2}}(w_1\sigma^2_{x_1}+w_2\sigma^2_{x_2}+(w_1+w_2)cov(x_1,x_2))\\
\times(w_1\sigma^2_{x_1}+w_2\sigma^2_{x_2}+(w_1+w_2)cov(x_1,x_2)) \\
=\frac{(\sigma^2_{x_1}\sigma^2_{x_2}-cov(x_1,x_2)^2)(w_1-w_2)^2}{\sigma^2_{x_1+x_2}}.
\end{gathered}
\end{equation*}
The result follows by definition of covariance.
\end{proof}

\begin{proof}[Proof of Theorem \ref{thm:diffFiniteBias}]
From Theorem \ref{thm:bias2D}, the difference between the one-dimensional and the two-dimensional bias is equal to the one-dimensional bias, since the two-dimensional model is unbiased. Therefore it is equal to:
\begin{equation*}
\begin{gathered}
\frac{\sigma^2_{x_1+x_2}(w_1\hat{\sigma}^2_{x_1}+w_2\hat{\sigma}^2_{x_2}+(w_1+w_2)\hat{cov}(x_1,x_2))^2}{(\hat{\sigma}^2_{x_1+x_2})^2}\\
+w_1^2\sigma^2_{x_1}+w_2^2\sigma^2_{x_2}+2w_1w_2cov(x_1,x_2)\\
-\frac{2(w_1\sigma^2_{x_1}+w_2\sigma^2_{x_2}+(w_1+w_2)cov(x_1,x_2))}{\hat{\sigma}^2_{x_1+x_2}}\\
\times(w_1\hat{\sigma}^2_{x_1}+w_2\hat{\sigma}^2_{x_2}+(w_1+w_2)\hat{cov}(x_1,x_2)).
\end{gathered}
\end{equation*}
Substituting the assumptions from Equation \eqref{eq:sameVar} we get:
\begin{equation*}
\begin{gathered}
\frac{2(\sigma^2_{x}+cov(x_1,x_2))((w_1+w_2)(\hat{cov}(x_1,x_2)+\hat{\sigma}^2_{x}))^2}{(\hat{\sigma}^2_{x_1+x_2})^2}\\
+\frac{(\hat{\sigma}^2_{x_1+x_2})^2(w_1^2\sigma^2_{x}+w_2^2\sigma^2_{x}+2w_1w_2cov(x_1,x_2))}{(\hat{\sigma}^2_{x_1+x_2})^2}+ \\
-\frac{2(\hat{\sigma}^2_{x_1+x_2})(w_1\sigma^2_{x}+w_2\sigma^2_{x}+(w_1+w_2)cov(x_1,x_2))}{(\hat{\sigma}^2_{x_1+x_2})^2}\\
\times(w_1\hat{\sigma}^2_{x}+w_2\hat{\sigma}^2_{x}+(w_1+w_2)\hat{cov}(x_1,x_2)),
\end{gathered}
\end{equation*}
from which, after basic algebraic computations, the result follows.
\end{proof}

\section{Two-dimensional analysis: additional setting}\label{app:addRes}
In the main paper, the finite-sample analysis in the two-dimensional case is performed with the assumption that the two features $x_1,x_2$ have respectively the same variance and sample variance. In this section are reported the results after the relaxation of this hypothesis.\\ 
In particular, the only assumption for this finite-sample analysis is unitary variance:
\begin{equation}\label{eq:unitVar}
    \sigma_{x_1}=\sigma_{x_2}=1,
\end{equation}
implying by definition of covariance that $cov(x_1,x_2)=\rho_{x_1,x_2}$.\\ 
This is not an impacting restriction since it is always possible to scale a random variable to have unitary variance dividing it by its standard deviation.
\\
The following theorem shows the difference of variance between the two-dimensional and the one-dimensional linear regression models.
\begin{theorem}
In the finite-case with unitary variances, the difference of variances of the linear model with two features compared to the one with a single feature (which is their mean), is equal to:
\begin{equation}\label{eq:general2Dvar}
\begin{gathered}\small
\frac{\sigma^2}{(n-1)}\times\Big[\frac{(\hat{\sigma}^2_{x_1}-\hat{\sigma}^2_{x_2})^2}{\hat{\sigma}^2_{x_1}\hat{\sigma}^2_{x_2}(1-\hat{\rho}_{x_1,x_2}^2)\hat{\sigma}^2_{x_1+x_2}}\\
\medmuskip=0mu
\thinmuskip=0mu
\thickmuskip=0mu
+\frac{2(1-\rho_{x_1,x_2})(\hat{\sigma}^2_{x_1}+\hat{cov}(x_1,x_2))(\hat{\sigma}^2_{x_2}+\hat{cov}(x_1,x_2))}{\hat{\sigma}^2_{x_1}\hat{\sigma}^2_{x_2}(1-\hat{\rho}_{x_1,x_2}^2)\hat{\sigma}^2_{x_1+x_2}}\Big].
\end{gathered}
\end{equation}
\end{theorem}
\begin{proof}
Starting from the difference of variances between the two-dimensional and the one-dimensional model (Theorem \ref{thm:variance2D}):
\begin{equation*}
\begin{gathered}
\frac{\sigma^2}{(n-1)(\hat{\sigma}^2_{x_1}\hat{\sigma}^2_{x_2} - \hat{cov}(x_1,x_2)^2)}\\
\times(\sigma^2_{x_1}\hat{\sigma}^2_{x_2}+\sigma^2_{x_2}\hat{\sigma}^2_{x_1}-2cov(x_1,x_2)\hat{cov}(x_1,x_2))\\
-\sigma_{x_1+x_2}^2\frac{\sigma^2}{(n-1)\hat{\sigma}^2_{x_1+x_2}}, 
\end{gathered}
\end{equation*}
exploiting the unitary variance assumption becomes:
\begin{equation*}
\begin{gathered}
\frac{\sigma^2}{(n-1)(\hat{\sigma}^2_{x_1}\hat{\sigma}^2_{x_2} - \hat{cov}(x_1,x_2)^2)}\\ \times(\hat{\sigma}^2_{x_2}+\hat{\sigma}^2_{x_1}-2cov(x_1,x_2)\hat{cov}(x_1,x_2))\\
- \frac{\sigma^2(2+2cov(x_1,x_2))}{(n-1)(\hat{\sigma}^2_{x_1}+\hat{\sigma}^2_{x_2}+2\hat{cov}(x_1,x_2))}.
\end{gathered}
\end{equation*}
Applying the common denominator:
\begin{equation*}
\begin{gathered}
\medmuskip=0mu
\thinmuskip=0mu
\thickmuskip=0mu
\frac{1}{(n-1)(\hat{\sigma}^2_{x_1}\hat{\sigma}^2_{x_2} - \hat{cov}(x_1,x_2)^2)(\hat{\sigma}^2_{x_1}+\hat{\sigma}^2_{x_2}+2\hat{cov}(x_1,x_2))}\\
\times\Bigg[\sigma^2(\hat{\sigma}^4_{x_1}+\hat{\sigma}^4_{x_2}+2\hat{\sigma}^2_{x_1}\hat{cov}(x_1,x_2)\\
+2\hat{\sigma}^2_{x_2}\hat{cov}(x_1,x_2)-2\hat{\sigma}^2_{x_1}\hat{cov}(x_1,x_2)cov(x_1,x_2))\\
+\sigma^2(-2\hat{\sigma}^2_{x_2}\hat{cov}(x_1,x_2)cov(x_1,x_2)\\
-2cov(x_1,x_2)\hat{cov}(x_1,x_2)^2)\\
+\sigma^2(2\hat{cov}(x_1,x_2)^2-2\hat{\sigma}^2_{x_1}\hat{\sigma}^2_{x_2}cov(x_1,x_2))\Bigg].
\end{gathered}
\end{equation*}

Finally, adding and subtracting on the numerator the term $2\hat{\sigma}^2_{x_1}\hat{\sigma}^2_{x_2}$ and grouping the terms, it is equal to:
\begin{equation*}
\begin{gathered}
\frac{\sigma^2}{(n-1)\hat{\sigma}^2_{x_1}\hat{\sigma}^2_{x_2}(1 - \hat{\rho}_{x_1,x_2}^2)\hat{\sigma}^2_{x_1+x_2}}\\
\times\Bigg[(\hat{\sigma}^2_{x_1}-\hat{\sigma}^2_{x_2})^2+2(1-\rho_{x_1,x_2})\\
\times(\hat{\sigma}^2_{x_1}+\hat{cov}(x_1,x_2))(\hat{\sigma}^2_{x_2}+\hat{cov}(x_1,x_2))\Bigg]. 
\end{gathered}
\end{equation*}
\end{proof}

\begin{remark}
When the number of samples $n$ tends to infinity the result becomes the same found in the asymptotic analysis. Moreover, when the sample variances of the two features are equal, the result becomes the same of the finite case analysis with equal sample and real variances.
\end{remark}

\begin{lemma}
The quantity found as difference of variances between the two-dimensional and one-dimensional case in this general setting is always non-negative.
\begin{proof}
Recalling the result of Equation \eqref{eq:general2Dvar}, the first factor $\frac{\sigma^2}{(n-1)}$ and the denominator of the second one $\hat{\sigma}^2_{x_1}\hat{\sigma}^2_{x_2}(1-\hat{\rho}_{x_1,x_2}^2)\hat{\sigma}^2_{x_1+x_2}$ are always non-negative, so the difference of features is positive if and only if the second numerator is positive:
\begin{equation*}
\begin{gathered}
(\hat{\sigma}^2_{x_1}-\hat{\sigma}^2_{x_2})^2+2(1-\rho_{x_1,x_2})\\
\times(\hat{\sigma}^2_{x_1}+\hat{cov}(x_1,x_2))(\hat{\sigma}^2_{x_2}+\hat{cov}(x_1,x_2))\geq0.
\end{gathered}
\end{equation*}

Focusing on the term $2(1-\rho_{x_1,x_2})(\hat{\sigma}^2_{x_1}+\hat{cov}(x_1,x_2))(\hat{\sigma}^2_{x_2}+\hat{cov}(x_1,x_2)),$
the function $(\hat{\sigma}^2_{x_1}+\hat{cov}(x_1,x_2))(\hat{\sigma}^2_{x_2}+\hat{cov}(x_1,x_2))$ takes minimum value when $\hat{cov}(x_1,x_2)=-\frac{\hat{\sigma}^2_{x_1}+\hat{\sigma}^2_{x_2}}{2}$, therefore the minimum value of this term is:
$$2(1-\rho_{x_1,x_2})(-\frac{1}{4})(\hat{\sigma}^2_{x_1}-\hat{\sigma}^2_{x_2})^2.$$
Substituting it back in the original inequality:
\begin{equation*}
\begin{gathered}
(\hat{\sigma}^2_{x_1}-\hat{\sigma}^2_{x_2})^2-\frac{1}{2}(1-\rho_{x_1,x_2})(\hat{\sigma}^2_{x_1}-\hat{\sigma}^2_{x_2})^2\\
=\frac{1}{2}(1+\rho_{x_1,x_2})(\hat{\sigma}^2_{x_1}-\hat{\sigma}^2_{x_2})^2,
\end{gathered}
\end{equation*}
that is a quantity always non-negative and proves the lemma.
\end{proof}
\end{lemma}

The following theorem shows the difference of (squared) bias between the one-dimensional and the two-dimensional models.
\begin{theorem}
In the finite-case, assuming unitary variances $\sigma_{x_1}=\sigma_{x_2}=1$, the increase of bias due to the aggregation of the two features with their average is equal to:
\begin{equation}
\begin{gathered}
\frac{1}{(\hat{\sigma}^2_{x_1+x_2})^2}\\
\times(2(1-\rho_{x_1,x_2})(\hat{\sigma}^2_{x_1}+\hat{\sigma}^2_{x_2}+\hat{cov}(x_1,x_2))\hat{cov}(x_1,x_2)\\
+\hat{\sigma}^4_{x_1}+\hat{\sigma}^4_{x_2}-2\rho_{x_1,x_2} \hat{\sigma}^2_{x_1}\hat{\sigma}^2_{x_2})(w_1-w_2)^2.
\end{gathered}
\end{equation}
\end{theorem}
\begin{proof}
Recalling the expression of the difference of bias between the one-dimensional and the two-dimensional linear regression models (Theorem \ref{thm:bias2D}):
\begin{equation*}
\begin{gathered}
\frac{\sigma^2_{x_1+x_2}}{(\hat{\sigma}^2_{x_1+x_2})^2}(w_1\hat{\sigma}^2_{x_1}+w_2\hat{\sigma}^2_{x_2}+(w_1+w_2)\hat{cov}(x_1,x_2))^2\\
+w_1^2\sigma^2_{x_1}+w_2^2\sigma^2_{x_2}+2w_1w_2cov(x_1,x_2)\\
-\frac{2}{\hat{\sigma}^2_{x_1+x_2}}\\
\times(w_1\sigma^2_{x_1}+w_2\sigma^2_{x_2}+(w_1+w_2)cov(x_1,x_2))\\
\times(w_1\hat{\sigma}^2_{x_1}+w_2\hat{\sigma}^2_{x_2}+(w_1+w_2)\hat{cov}(x_1,x_2)),
\end{gathered}
\end{equation*}
exploiting the unitary variance assumption can be written as:
\begin{equation*}
\begin{gathered}
\frac{2(1+\rho_{x_1,x_2})}{(\hat{\sigma}^2_{x_1+x_2})^2}(w_1\hat{\sigma}^2_{x_1}+w_2\hat{\sigma}^2_{x_2}+(w_1+w_2)\hat{cov}(x_1,x_2))^2\\
+\frac{(w_1^2+w_2^2+2w_1w_2\rho_{x_1,x_2})(\hat{\sigma}^2_{x_1}+\hat{\sigma}^2_{x_2}+2\hat{cov}(x_1,x_2))^2}{(\hat{\sigma}^2_{x_1+x_2})^2}\\
-\frac{2}{(\hat{\sigma}^2_{x_1+x_2})^2}(w_1+w_2)(1+\rho_{x_1,x_2})\\
\times(w_1\hat{\sigma}^2_{x_1}+w_2\hat{\sigma}^2_{x_2}+(w_1+w_2)\hat{cov}(x_1,x_2))\\
\times(\hat{\sigma}^2_{x_1}+\hat{\sigma}^2_{x_2}+2\hat{cov}(x_1,x_2)).
\end{gathered}
\end{equation*}
After basic algebraic computations the result follows.
\end{proof}

\begin{remark}
When the number of samples $n$ tends to infinity the result becomes the same found in the asymptotic analysis. Moreover, when the sample variances of the two features are equal, the result becomes the same of the finite case analysis with equal sample and real variances.
\end{remark}

\begin{theorem}
Necessary and sufficient condition for positivity of the difference between the reduction of variance and the increase of bias when aggregating two features with their average in the unitary-variance finite-sample setting is:
\begin{equation}
    \begin{gathered}
    \sigma^2[(\hat{\sigma}^2_{x_1}-\hat{\sigma}^2_{x_2})^2\\
    +2(1-\rho_{x_1,x_2})(\hat{\sigma}^2_{x_1}+\hat{cov}(x_1,x_2))(\hat{\sigma}^2_{x_2}+\hat{cov}(x_1,x_2))]\\
    \times(\hat{\sigma}^2_{x_1}+\hat{\sigma}^2_{x_2}+2\hat{cov}(x_1,x_2))\\
     - [2(1-\rho_{x_1,x_2})(\hat{\sigma}^2_{x_1}+\hat{\sigma}^2_{x_2}+\hat{cov}(x_1,x_2))\hat{cov}(x_1,x_2)\\
     +\hat{\sigma}^4_{x_1}+\hat{\sigma}^4_{x_2}-2\rho\hat{\sigma}^2_{x_1}\hat{\sigma}^2_{x_2}]\\
     \times(w_1-w_2)^2(n-1)(\hat{\sigma}^2_{x_1}\hat{\sigma}^2_{x_2}-\hat{cov}(x_1,x_2)^2) \geq 0.
     \end{gathered}
    \end{equation}
\end{theorem}

\begin{proof}
The result is obtained subtracting the results of the two previous theorems, after algebraic computations.
\end{proof}

\section{Two-dimensional analysis: Theoretical and Practical quantities}\label{subsec:confInt}
This section elaborates the inequalities found in the main paper in Theorem \ref{thm:asympBound2D}, \ref{thm:finiteBound2D} considering only theoretical quantities or, on the other hand, quantities that can all be computed from data. In this way, in the bivariate case, we have both a theoretical conclusion of the analysis and an empirical one that can be used in practice.\\ \ \\
For the asymptotic analysis it is straightforward to obtain a theoretical and an empirical expression, indeed at the limit the estimators converge in probability to the theoretical quantities.
\begin{theorem}
\label{thm:asymptTeoEmp}
In the asymptotic setting of Theorem \ref{thm:asympBound2D}, considering only theoretical quantities, the following inequalities hold:
\begin{equation}\label{eq:asympBound2Dteo}
\begin{cases}\rho^2_{x_1,x_2} \geq 1-\frac{\sigma^2\sigma^2_{x_1+x_2}}{(n-1)\sigma^2_{x_1}\sigma^2_{x_2}(w_1-w_2)^2}\\
\rho_{x_1,x_2} \geq 1-\frac{2\sigma^2}{(n-1)(w_1-w_2)^2}\ (if\ \sigma_{x_1}=\sigma_{x_2}=1).\end{cases}
\end{equation}
On the other hand, considering only quantities that can be derived from data:
\begin{equation}\label{eq:asympBound2Demp}
\begin{cases}\hat{\rho}^2_{x_1,x_2} \geq 1-\frac{s^2\hat{\sigma}^2_{x_1+x_2}}{(n-1)\hat{\sigma}^2_{x_1}\hat{\sigma}^2_{x_2}(\hat{w}_1-\hat{w}_2)^2} \\
\hat{\rho}_{x_1,x_2} \geq 1-\frac{2s^2}{(n-1)(\hat{w}_1-\hat{w}_2)^2}\ (if\ \hat{\sigma}_{x_1}=\hat{\sigma}_{x_2}=1),
\end{cases}
\end{equation}
where $s^2=\frac{\hat{\epsilon}^T\hat{\epsilon}}{n-3}$ is the unbiased estimator of the variance $\sigma^2$ of the residual $\epsilon$ of the linear regression (an estimate $\hat{\epsilon}$ of the residual can be computed subtracting the predicted value to the real value of the target).
\end{theorem}
\begin{proof}
Equation \eqref{eq:asympBound2Dteo} is the same result of Theorem \ref{thm:asympBound2D}.\\
To derive Equation \eqref{eq:asympBound2Demp} it is sufficient to substitute the theoretical quantities with their consistent estimators.
\end{proof}

For the finite-samples analysis it is necessary to introduce confidence intervals to substitute theoretical with empirical quantities and viceversa.

\begin{theorem}
\label{thm:finBound2Demp}
In the finite-case setting of Theorem \ref{thm:finiteBound2D}, considering only empirical quantities, the following inequality holds with probability at least $1-\delta$:
\begin{equation}
\begin{gathered}
\label{eq:finBound2Demp}
\hat{\rho}_{x_1,x_2} \geq
1-\frac{2(n-3)s^2}{(n-1)\chi^2_{n-3}(\frac{\delta}{2})\hat{\sigma}^2_x}\\
\medmuskip=0mu
\thinmuskip=0mu
\thickmuskip=0mu
\times\frac{1}{(\lvert\hat{w}_1-\hat{w}_2\lvert +\sqrt{3F_{3,n-3}(\frac{\delta}{2})}(\sqrt{\hat{var}(\hat{w}_1)}+\sqrt{\hat{var}(\hat{w}_2)}))^2},
\end{gathered}
\end{equation}
where $\chi^2_{n-3}(\cdot)$ represents a Chi-squared distribution with $n-3$ degrees of freedom and $F_{3,n-3}(\cdot)$ a Fisher distribution with $3,n-3$ degrees of freedom.
\begin{proof}
The unilateral confidence interval for the variance $\sigma^2$ of the residual $\epsilon$ of the linear regression model $y=w_1x_1+w_2x_2+\epsilon$, assuming $\epsilon \sim \mathcal{N} (0,\sigma^2)$ is, with probability $1-\alpha$~\citep{johnson2007}: $$\frac{(n-r-1)s^2}{\chi^2_{n-r-1}(\alpha)}\leq \sigma^2.$$
The simultaneous confidence interval for the weights $w_1,w_2$ of the linear regression model $y=w_1x_1+w_2x_2+\epsilon$ is, with probability $1-\gamma$: $$\begin{cases}
w_1\in [\hat{w_1}\pm \sqrt{\hat{var}(\hat{w}_1)}\sqrt{(r+1)F_{r+1,n-r-1}(\gamma)}]\\
w_2\in [\hat{w_2}\pm \sqrt{\hat{var}(\hat{w}_2)}\sqrt{(r+1)F_{r+1,n-r-1}(\gamma)}].
\end{cases}$$
Considering the confidence intervals and the inequality of Theorem \ref{thm:finiteBound2D}, with probability $1-\delta$:
\begin{equation*}
\begin{gathered}
\hat{\rho}_{x_1,x_2} \geq
1-\frac{2(n-3)s^2}{(n-1)\chi^2_{n-3}(\frac{\delta}{2})\hat{\sigma}^2_x}\\
\times\frac{1}{(\lvert\hat{w}_1-\hat{w}_2\rvert+\sqrt{3F_{3,n-3}(\frac{\delta}{2})}(\sqrt{\hat{var}(\hat{w}_1)}+\sqrt{\hat{var}(\hat{w}_2)}))^2}\\
\geq 1-\frac{2\sigma^2}{(n-1)\hat{\sigma}^2_x(w_1-w_2)^2}.
\end{gathered}
\end{equation*}

This means that the inequality holds and concludes the proof.
\end{proof}
\end{theorem}

\begin{remark}
At the limit the quantity $\frac{\chi^2}{n-3}$ tends to 1, so the result of Theorem \ref{thm:finBound2Demp} is coherent with the asymptotic result. 
\end{remark}

In order to obtain the result with only theoretical quantities in the finite-sample case, it is necessary to introduce two bounds on the difference between covariance and sample covariance.

\begin{proposition}\label{prop:covIneq}
The following inequalities hold.
\begin{itemize}
    \item With probability $1-\delta$:
    \begin{equation}\label{eq:covUpper}
    \hat{cov}(x_1,x_2)-cov(x_1,x_2)\leq3\sqrt{\frac{log(\frac{4}{\delta})}{n-1}}.
    \end{equation}
\item with probability $1-\delta$:
    \begin{equation}\label{eq:covLower}
cov(x_1,x_2)-\hat{cov}(x_1,x_2)\leq4\sqrt{\frac{log(\frac{4}{\delta})}{n-1}};
    \end{equation}
\end{itemize}\end{proposition}
\begin{proof}
The proof exploits Hoeffding's inequality by applying it to the random variable $Z=x_1x_2$. 
The proof will derive the results of the proposition for two general random variables $X,Y$ and $n$ data $x_i,y_i$ sampled from their distribution. We denote with $\bar{X},\bar{Y}$ the means of the considered samples.\\
From Hoeffding's inequality~\citep{Hoeffding1963} applied to the variable $Z=XY$, with probability $1-\delta$, we get:
\begin{equation*}
\lvert\E[XY]-\frac{1}{n}\sum_{i=1}^{n}x_iy_i\rvert\leq\sqrt{\frac{log(\frac{2}{\delta})}{2n}},
\end{equation*}
which implies: \begin{equation*}\lvert\E[XY]-\frac{1}{n-1}\sum_{i=1}^{n}x_iy_i\rvert\leq\sqrt{\frac{log(\frac{2}{\delta})}{n-1}}.
\end{equation*}
Then with probability $1-2\delta$:
\begin{equation*}
\begin{gathered}
\hat{cov}(X,Y)-cov(X,Y)\\
=\frac{1}{n-1}\sum_{i=1}^{n}x_iy_i - \frac{n}{n-1}\bar{X}\bar{Y}-\E[XY]+\E[X]\E[Y]\\
\leq\frac{1}{n-1}\sum_{i=1}^{n}x_iy_i - \bar{X}\bar{Y}-\E[XY]+\E[X]\E[Y]\pm\bar{X}\E[Y]\\
\leq\sqrt{\frac{log(\frac{2}{\delta})}{n-1}}+\bar{Y}\sqrt{\frac{log(\frac{2}{\delta})}{2n}}+\E[{Y}]\sqrt{\frac{log(\frac{2}{\delta})}{2n}}\leq3\sqrt{\frac{log(\frac{2}{\delta})}{n-1}}, 
\end{gathered}
\end{equation*}
where the second inequality applies Hoeffding's inequality three times. Equation \eqref{eq:covUpper} is therefore proved.\\
On the other hand, with probability $1-2\delta$:
\begin{equation*}
    \begin{gathered}
cov(X,Y)-\hat{cov}(X,Y)\\
= -\frac{1}{n-1}\sum_{i=1}^{n}x_iy_i + \frac{n}{n-1}\bar{X}\bar{Y}+\E[XY]-\E[X]\E[Y]\\
=-\frac{1}{n-1}\sum_{i=1}^{n}x_iy_i + \bar{X}\bar{Y}+\frac{1}{n-1}\bar{X}\bar{Y}\\
+\E[XY]-\E[X]\E[Y]\pm\bar{X}\E[Y]\\
\leq\sqrt{\frac{log(\frac{2}{\delta})}{n-1}}+\bar{X}\sqrt{\frac{log(\frac{2}{\delta})}{2n}}+\E[{Y}]\sqrt{\frac{log(\frac{2}{\delta})}{2n}}+\frac{1}{n-1}\bar{X}\bar{Y}\\
\leq4\sqrt{\frac{log(\frac{2}{\delta})}{n-1}}, 
    \end{gathered}
\end{equation*}
where the first inequality is again due to the application of Hoeffding's inequality three times. From this result, Equation \ref{eq:covLower} follows.
\end{proof}

It is now possible to derive the expression of Equation \eqref{eq:2DfinBound} with only theoretical quantities.

\begin{theorem}
\label{thm:finBound2Dteo}
In the finite-case setting of Theorem \ref{thm:finiteBound2D}, considering only theoretical quantities, the following inequality holds with probability at least $1-\delta$:
\begin{equation}
\begin{gathered}
\label{eq:finBound2Dteo}
\rho_{x_1,x_2} \geq 1-\frac{2\sigma^2}{(n-1)\sigma^2_x(w_1-w_2)^2}\\
+\frac{1}{\sigma_x^2}\Bigg( \frac{2log(\frac{2}{\delta})}{n-1} + 2\sigma_x\sqrt{\frac{2log(\frac{2}{\delta})}{n-1}}+4 \sqrt{\frac{log(\frac{8}{\delta})}{n-1}} \Bigg).
\end{gathered}
\end{equation}
\end{theorem}
\begin{proof}
To prove the theorem it is enough to apply the upper bound for the sample variance from~\citep{Maurer2009} and the lower bound for the sample covariance from the inequalities of Proposition \ref{prop:covIneq}.\\
Regarding the sample variance, with probability $1-\alpha$ it holds~\citep{Maurer2009}: 
\begin{equation*}
    \hat{\sigma}^2_x\leq \Bigg(\sigma_x+\sqrt{\frac{2log(\frac{1}{\alpha})}{n-1}}\Bigg)^2.
\end{equation*}
For the sample covariance, Proposition \ref{prop:covIneq} shows that with probability $1-\gamma$:
\begin{equation*}
\hat{cov}(x_1,x_2)\geq cov(x_1,x_2)-4\sqrt{\frac{log(\frac{4}{\gamma})}{n-1}}.
\end{equation*}
Starting from the inequality of Theorem \ref{thm:finiteBound2D}:
\begin{equation*}
\hat{\rho}_{x_1,x_2} \geq 1-\frac{2\sigma^2}{(n-1)\hat{\sigma}^2_x(w_1-w_2)^2},
\end{equation*}
it is equal to:
\begin{equation*}
\hat{\sigma}^2_x-\hat{cov}(x_1,x_2) 
\leq \frac{2\sigma^2}{(n-1)(w_1-w_2)^2}.
\end{equation*}
Therefore, with probability $1-\delta$:
\begin{equation*}
\begin{gathered}
\hat{\sigma}^2_x-\hat{cov}(x_1,x_2) \\
\leq \Bigg(\sigma_x+\sqrt{\frac{2log(\frac{2}{\delta})}{n-1}}\Bigg)^2-
\Bigg(cov(x_1,x_2)-4\sqrt{\frac{log(\frac{8}{\delta})}{n-1}}\Bigg)\\
\leq \frac{2\sigma^2}{(n-1)(w_1-w_2)^2}.
\end{gathered}
\end{equation*}
After basic algebraic computations, the result follows.
\end{proof}

\begin{remark}
The empirical result of Theorem \ref{thm:finBound2Demp} depends on the distribution of the residual, assuming it to be Gaussian. On the other hand, the theoretical expression of Theorem \ref{thm:finBound2Dteo} does not need any assumption on the distribution.
\end{remark}

\section{Three-dimensional algorithm}\label{app:proofs3D}
This section contains detailed results and proofs related to Section \ref{sec:3andD} of the main paper.

\paragraph{bias of the two models, expressions and derivations}
In the asymptotic setting, letting the relationship between the features and the target be linear with Gaussian noise and assuming unitary variances of the features $\sigma_{x_1}=\sigma_{x_2}=\sigma_{x_3}=1$, for the one-dimensional regression $\hat{y}=\hat{w}\frac{x_1+x_2}{2}$:
\begin{equation*}
\begin{aligned}
\E_x&[(\Bar{h}(x)-\bar{y})^2] \\ &=\mathcal{F}(x_1,x_2,x_3,w_1,w_2,w_3,\rho_{x_1,x_2},\rho_{x_1,x_3},\rho_{x_2,x_3})\\
&=-\frac{((w_1+w_2)(1+\rho_{x_1,x_2})+w_3(\rho_{x_1,x_3}+\rho_{x_2,x_3}))^2}{2(1+\rho_{x_1,x_2})}\\
& \quad +\E_x[(w_1x_1+w_2x_2+w_3x_3)^2].
\end{aligned}
\end{equation*}
For the two-dimensional regression $\hat{y}=\hat{w_1}x_1+\hat{w_2}x_2$:
\begin{equation*}
\begin{aligned}
 \E_x&[(\Bar{h}(x)-\bar{y})^2] \\
 &= \mathcal{G}(x_1,x_2,x_3,w_1,w_2,w_3,\rho_{x_1,x_2},\rho_{x_1,x_3},\rho_{x_2,x_3})\\
 &=(w_1+aw_3)^2+(w_2+bw_3)^2\\
&+2\rho_{x_1,x_2}(w_1+aw_3)(w_2+bw_3)\\
&-2(w_1+aw_3)(w_1+w_2\rho_{x_1,x_2}+w_3\rho_{x_1,x_3})\\
&-2(w_2+bw_3)(w_1\rho_{x_1,x_2}+w_2+w_3\rho_{x_2,x_3})\\
& +\E_x[(w_1x_1+w_2x_2+w_3x_3)^2],\\
& \text{with } 
\begin{cases}
a=\frac{\rho_{x_1,x_3}-\rho_{x_1,x_2}\rho_{x_2,x_3}}{1-\rho_{x_1,x_2}^2}\\ b=\frac{\rho_{x_2,x_3}-\rho_{x_1,x_2}\rho_{x_1,x_3}}{1-\rho_{x_1,x_2}^2}. 
\end{cases}
\end{aligned}
\end{equation*}

\begin{proof}
In the one dimensional setting, letting $\bar{x}=\frac{x_1+x_2}{2}$:
\begin{equation*}
\begin{gathered}
\E_x[(\Bar{h}(x)-\bar{y})^2]\\
=\E_{x}[(\E_\mathcal{T}[\hat{w}\bar{x}]-(w_1x_1+w_2x_2+w_3x_3))^2]\\
=\sigma^2_{\bar{x}}\E_\mathcal{T}[\hat{w}]^2+\E_{x}[(w_1x_1+w_2x_2+w_3x_3)^2]\\
-2\E_\mathcal{T}[\hat{w}]^2\E_{x}[\bar{x}(w_1x_1+w_2x_2+w_3x_3)].
\end{gathered}
\end{equation*}

Where the last equivalence is due to the fact that train and test data are independent. Then, conditioning on the training features set $\mathbf{X}$ and substituting the value of $\E_\mathcal{T}[\hat{w}\lvert\mathbf{X}]$ from Equation \eqref{eq:ExpvalEst}, if follows:
\begin{equation*}
\begin{gathered}
\E_x[(\Bar{h}(x)-\bar{y})^2\lvert\mathbf{X}]\\ =\frac{\sigma^2_{x_1+x_2}}{(\hat{\sigma}^2_{x_1+x_2})^2}\\
\times(w_1\hat{\sigma}^2_{x_1}+w_2\hat{\sigma}^2_{x_2}+(w_1+w_2)\hat{cov}(x_1,x_2)\\
+w_3(\hat{cov}(x_1,x_3)+\hat{cov}(x_2,x_3)))^2\\ +\E_x[(w_1x_1+w_2x_2+w_3x_3)^2]\\
-\frac{2}{\hat{\sigma}^2_{x_1+x_2}}(w_1\hat{\sigma}^2_{x_1}+w_2\hat{\sigma}^2_{x_2}+(w_1+w_2)\hat{cov}(x_1,x_2)\\
+w_3(\hat{cov}(x_1,x_3)+\hat{cov}(x_2,x_3)))\\
\times(w_1\sigma^2_{x_1}+w_2\sigma^2_{x_2}+(w_1+w_2)cov(x_1,x_2)\\
+w_3(cov(x_1,x_3)+cov(x_2,x_3))).
\end{gathered}
\end{equation*}
Considering the asymptotic case, substituting the sample estimators with the real statistical measures and the variances with 1 the result follows.\\
In the two dimensional setting:
\begin{equation*}
\begin{gathered}
\E_x[(\Bar{h}(x)-\bar{y})^2]\\
=\E_{x}[(\E_\mathcal{T}[\hat{w_1}x_1+\hat{w_2}x_2]-(w_1x_1+w_2x_2+w_3x_3))^2]\\
=\sigma^2_{x_1}\E_\mathcal{T}[\hat{w_1}]^2+\sigma^2_{x_2}\E_\mathcal{T}[\hat{w_2}]^2+2cov(x_1,x_2)\E_\mathcal{T}[\hat{w_1}]\E_\mathcal{T}[\hat{w_2}]\\
+\E_x[(w_1x_1+w_2x_2+w_3x_3)^2]\\ -2(\E_\mathcal{T}[\hat{w_1}]\E_x[x_1(w_1x_1+w_2x_2+w_3x_3)]\\+\E_\mathcal{T}[\hat{w_2}]\E_x[x_1(w_1x_1+w_2x_2+w_3x_3)]),
\end{gathered}
\end{equation*}
exploiting again the independence between train and test data.\\ Then, conditioning on the training set $\mathbf{X}$, substituting the value of $\E_\mathcal{T}[\hat{w}\lvert\mathbf{X}]^2$ from Equation \eqref{eq:ExpvalEst} and calling:
\begin{equation*}
\begin{cases}
a=\frac{\hat{\sigma}^2_{x_2}\hat{cov}(x_1,x_3)-\hat{cov}(x_1,x_2)\hat{cov}(x_2,x_3)}{\hat{\sigma}^2_{x_1}\hat{\sigma}^2_{x_2}-\hat{cov}(x_1,x_2)^2}\\ b=\frac{\hat{\sigma}^2_{x_1}\hat{cov}(x_2,x_3)-\hat{cov}(x_1,x_2)\hat{cov}(x_1,x_3)}{\hat{\sigma}^2_{x_1}\hat{\sigma}^2_{x_2}-\hat{cov}(x_1,x_2)^2}
\end{cases},
\end{equation*}
it follows:
\begin{equation*}
\begin{gathered}
\E_x[(\Bar{h}(x)-\bar{y})^2\lvert\mathbf{X}] = \\
\sigma^2_{x_1}(w_1+w_3a)^2+\sigma^2_{x_2}(w_2+w_3b)^2\\
+2cov(x_1,x_2)(w_1+w_3a)(w_2+w_3b)\\
+\E_x[(w_1x_1+w_2x_2+w_3x_3)^2]\\
-2((w_1+w_3a)(w_1\sigma^2_{x_1}+w_2cov(x_1,x_2)+w_3cov(x_1,x_3))\\
+(w_2+w_3b)(w_1cov(x_1,x_2)+w_2\sigma^2_{x_2}+w_3cov(x_2,x_3))).
\end{gathered}
\end{equation*}
For the asymptotic case, substituting the sample estimators with the real statistical measures and the variances with 1 the result follows.\\
\end{proof}

\paragraph{Extension of bias of the models to general variance of third variable}
Considering a general variance $\sigma^2_{x_3}$ for the third variable $x_3$, the bias of the models computed in the previous paragraph become (respectively for the one and two dimensional estimates):
\begin{equation*}
\begin{aligned}
\E_x&[(\Bar{h}(x)-\bar{y})^2] \\
& =  -\frac{((w_1+w_2)(1+\rho_{x_1,x_2})+w_3\sigma_{x_3}(\rho_{x_1,x_3}+\rho_{x_2,x_3}))^2}{2(1+\rho_{x_1,x_2})}
\\& \quad +\E_x[(w_1x_1+w_2x_2+w_3x_3)^2],\\
\E_x & [(\Bar{h}(x)-\bar{y})^2] = (w_1+aw_3)^2+(w_2+bw_3)^2\\
&+2\rho_{x_1,x_2}(w_1+aw_3)(w_2+bw_3)\\
&-2(w_1+aw_3)(w_1+w_2\rho_{x_1,x_2}+w_3\rho_{x_1,x_3}\sigma_{x_3})\\
&-2(w_2+bw_3)(w_1\rho_{x_1,x_2}+w_2+w_3\rho_{x_2,x_3}\sigma_{x_3})\\
&+\E_x[(w_1x_1+w_2x_2+w_3x_3)^2],\\
&\text{with } \begin{cases}
a=\sigma_{x_3}\frac{\rho_{x_1,x_3}-\rho_{x_1,x_2}\rho_{x_2,x_3}}{1-\rho_{x_1,x_2}^2}\\ b=\sigma_{x_3}\frac{\rho_{x_2,x_3}-\rho_{x_1,x_2}\rho_{x_1,x_3}}{1-\rho_{x_1,x_2}^2}.  
\end{cases}
\end{aligned}
\end{equation*}

\section{Experiments}\label{app:exp}

This section provides more details and results on the experiments performed in the two-dimensional, three-dimensional and $D$-dimensional settings.

\subsection{Bivariate synthetic data}\label{subs:2dimSyn}

\begin{table*}[th]
\caption{95\% confidence intervals for bivariate synthetic experiments with large difference of weights ($w_1=0.2,w_2=0.8$)
\label{tab:2dimSyn}}
\centering
\resizebox{\textwidth}{!}
{\begin{tabular}{@{}|c|ccc|@{}}
\hline 
    & \multicolumn{3}{|c|}{Confidence Interval for different Standard deviations of the noise}\\\hline 
Quantity & $\sigma=0.5\ (\sigma^2=0.25)$ & $\sigma=1\ (\sigma^2=1)$ & $\sigma=10\ (\sigma^2=100)$ \\\hline 
Theoretical $\bar{\rho}$ & $0.997217$ & $0.988867$ & $-0.113338$ \\
Empirical $\bar{\rho}$ (median) & $0.997218$ & $0.988807$ & $0.819620$ \\
$\hat{\rho}_{x_1,x_2}$ & $0.919045\pm1.78\mathrm{e}{-4}$ & $0.918558\pm1.66\mathrm{e}{-4}$ & $0.918901\pm1.7\mathrm{e}{-4}$ \\ 
$\hat{\sigma}^2$ & $0.250211\pm5.31\mathrm{e}{-4}$ & $0.999966\pm2.29\mathrm{e}{-4}$ & $100.354737\pm0.230573$\\
$\hat{w}_1$ & $0.198552\pm2.094\mathrm{e}{-3}$ & $0.200106\pm4.124\mathrm{e}{-3}$ & $0.198441\pm0.041063$\\ 
$\hat{w}_2$ & $0.800265\pm2.084\mathrm{e}{-3}$ & $0.800015\pm4.151\mathrm{e}{-3}$ & $0.783078\pm0.040011$\\ 
$\hat{w}$ & $0.998817\pm8.33\mathrm{e}{-4}$ & $1.000121\pm1.678\mathrm{e}{-3}$ & $0.981519\pm0.015914$\\ 
\hline
$R^2$ full & $0.781094\pm5.2\mathrm{e}{-5}$ & $0.487304\pm1.23\mathrm{e}{-4}$ & $0.010205\pm2.64\mathrm{e}{-4}$\\
$R^2$ aggr & $0.764944\pm3\mathrm{e}{-5}$ & $0.485527\pm7.8\mathrm{e}{-5}$ & $0.010630\pm1.69\mathrm{e}{-4}$\\
\textit{MSE} full & $0.275187\pm6.5\mathrm{e}{-5}$ & $1.000209\pm2.40\mathrm{e}{-4}$ & $103.020861\pm0.027468$\\
\textit{MSE} aggr & $0.295489\pm3.7\mathrm{e}{-5}$ & $1.003677\pm1.52\mathrm{e}{-4}$ & $102.976615\pm0.017587$\\
Var full & $0.001507\pm2.23\mathrm{e}{-4}$ & $0.006333\pm9.02\mathrm{e}{-4}$ & $0.590507\pm0.082236$\\
Var aggr & $0.001005\pm1.50\mathrm{e}{-4}$ & $0.004108\pm5.66\mathrm{e}{-4}$ & $0.383723\pm0.056062$\\
Bias full & $0.273680\pm 3.389\mathrm{e}{-3}$ & $0.993876\pm1.281\mathrm{e}{-3}$ & $102.430354\pm13.192218$\\
Bias aggr & $0.294484\pm3.773\mathrm{e}{-3}$ & $0.999569\pm1.270\mathrm{e}{-3}$ & $102.592893\pm13.165301$\\
\hline 
Aggregations (theo) & $\mathbf{0}$ & $\mathbf{0}$ & $\mathbf{500}$ \\
Aggregations (emp) & $\mathbf{0}$ & $\mathbf{24}$ & $\mathbf{332}$ \\
\hline
\end{tabular}}
\end{table*}

\begin{table*}[th]
\caption{95\% confidence intervals for bivariate synthetic experiments with small difference of weights ($w_1=0.47,w_2=0.52$)\label{tab:2dimSynSmall}}
\centering
\resizebox{\textwidth}{!}
{\begin{tabular}{@{}|c|ccc|@{}} \hline 
    & \multicolumn{3}{|c|}{Confidence Interval for different Standard deviations of the noise}\\\hline 
Quantity & $\sigma=0.5\ (\sigma^2=0.25)$ & $\sigma=1\ (\sigma^2=1)$ & $\sigma=10\ (\sigma^2=100)$ \\\hline 
Theoretical $\bar{\rho}$ & $0.599198$ & $-0.603206$ & $-25.675559$ \\
Empirical $\bar{\rho}$ (median) & $0.866424$ & $0.813952$ & $0.802706$ \\
$\hat{\rho}_{x_1,x_2}$ & $0.919192\pm1.78\mathrm{e}{-4}$ & $0.918774\pm1.81\mathrm{e}{-4}$ & $0.919326\pm1.72\mathrm{e}{-4}$ \\ 
$\hat{\sigma}^2$ & $0.250829\pm5.65\mathrm{e}{-4}$ & $1.003768\pm2.276\mathrm{e}{-4}$ & $99.955617\pm0.225969$\\
$\hat{w}_1$ & $0.467944\pm2.064\mathrm{e}{-3}$ & $0.466161\pm3.907\mathrm{e}{-3}$ & $0.370056\pm0.039486$\\ 
$\hat{w}_2$ & $0.522283\pm2.079\mathrm{e}{-3}$ & $0.526410\pm3.959\mathrm{e}{-3}$ & $0.634829\pm0.039561$\\ 
$\hat{w}$ & $0.990227\pm8.15\mathrm{e}{-4}$ & $0.992571\pm1.646\mathrm{e}{-3}$ & $1.004884\pm1.634\mathrm{e}{-3}$\\ 
\hline
$R^2$ full & $0.793788\pm6.4\mathrm{e}{-5}$ & $0.505334\pm9.1\mathrm{e}{-5}$ & $0.012507\pm2.31\mathrm{e}{-4}$\\
$R^2$ aggr & $0.794190\pm6.0\mathrm{e}{-5}$ & $0.506237\pm7.5\mathrm{e}{-5}$ & $0.014699\pm2.0\mathrm{e}{-5}$\\
\textit{MSE} full & $0.262609\pm8.2\mathrm{e}{-5}$ & $0.948725\pm1.75\mathrm{e}{-4}$ & $94.572608\pm2.209\mathrm{e}{-3}$\\
\textit{MSE} aggr & $0.262098\pm7.7\mathrm{e}{-5}$ & $0.946993\pm1.45\mathrm{e}{-4}$ & $94.362608\pm1.914\mathrm{e}{-3}$\\
Var full & $0.001524\pm2.16\mathrm{e}{-4}$ & $0.005719\pm8.22\mathrm{e}{-4}$ & $0.571782\pm0.080326$\\
Var aggr & $0.000998\pm1.44\mathrm{e}{-4}$ & $0.003997\pm5.76\mathrm{e}{-4}$ & $0.371541\pm0.054317$\\
Bias full & $0.261085\pm0.034614$ & $0.943006\pm0.127423$ & $94.000826\pm11.836568$\\
Bias aggr & $0.261105\pm0.034928$ & $0.943996\pm0.127362$ & $93.991098\pm11.826050$\\
\hline 
Aggregations (theo) & $\mathbf{500}$ & $\mathbf{500}$ & $\mathbf{500}$ \\
Aggregations (emp) & $\mathbf{314}$ & $\mathbf{339}$ & $\mathbf{346}$ \\
\hline
\end{tabular}}

\end{table*}

As introduced in Section \ref{sec:appl} of the main paper, the experiments performed in the bivariate setting are synthetic. In particular, for each of the six experiments, the data are computed as follows. The samples of the first independent variable $x_1$ are extracted from a uniform distribution in the interval $[0,1]$. The second feature $x_2$ is a linear combination between the feature $x_1$ and a random sample extracted from a uniform distribution in the interval $[0,1]$ (specifically $x_2=0.8x_1+0.2u,\ u\sim\mathcal{U}([0,1])$). Finally, the target variable $y$ is a linear combination between the two features $x_1,x_2$ with weights $w_1,w_2$ and the addition of a gaussian noise with variance $\sigma^2$.\\
Table \ref{tab:2dimSyn},\ref{tab:2dimSynSmall} provide more details about the bivariate results introduced in Table \ref{tab:bidimSmall} in the main paper.\\
In Table \ref{tab:2dimSyn} the extended results associated with large difference between weights $w_1=0.2, w_2=0.8$ and three different values of standard deviation of the noise $\sigma\in\{0.5,1,10\}$ are reported, repeating $s=500$ times each experiment, considering $n=500$ data for training and $n=500$ data for testing. The quantity $\bar{\rho}$ represents the minimum value of correlation for which it is convenient to aggregate the two features and it is computed exploiting the asymptotic result of Equation \eqref{eq:2DasymBuondRed}. From its theoretical values is clear that, for a reasonable amount of variance of the noise, since the weights of the linear model are significantly different, it is better to keep the features separated. This is confirmed by the confidence intervals of the $R^2$ and the $MSE$, which are both better in the bivariate case (full) rather than the univariate case (aggr). On the other hand, when the variance of the noise is large, the process becomes much more noisy and it is convenient to aggregate the two features. Considering the empirical $\bar{\rho}$ rather than the theoretical one, which is unknown with real datasets, the only situation where in the majority of the cases it is useful to aggregate is with large variance. It is important also to notice that the confidence intervals are much larger in the noisy setting, meaning that there is much more uncertainty and therefore it is useful to aggregate the two features in order to reduce it.\\
In Table \ref{tab:2dimSynSmall} the same results are reported, considering a linear relationship with small difference between weights $w_1=0.47,w_2=0.52$. In this setting, since the weights are closer, also with small amount of variance it is convenient to aggregate the two features if they are sufficiently correlated. Also with the empirical evaluation of the threshold, the two features would be aggregated for the majority of the repetitions of the experiment, leading to a non-worsening of the $MSE$ and $R^2$ coefficient for the aggregated case, as shown by the confidence intervals. 

\subsection{Three-dimensional synthetic data}\label{subs:3dimSyn}

\begin{table}[h]
\caption{Detailed Synthetic experiment in the three dimensional setting.
\label{tab:3dimSynExtended}}
\centering 
{\begin{tabular}{@{}|cc|@{}} \hline 
Quantity & 95\% Confidence Interval \\\hline 
Theoretical $\bar{\rho}$ (lower,upper) & $0.826063,\ 0.932936$ \\
Empirical $\bar{\rho}$ (median) & $0.831487,\ 0.933358$\\
$\hat{\rho}_{x_1,x_2}$ & $0.880300\pm1.07\mathrm{e}{-4}$ \\ 
$\hat{\sigma}^2$ & $0.249982\pm2.29\mathrm{e}{-4}$\\
$\hat{w}_1$ & $0.399401\pm9.12\mathrm{e}{-4}$\\ 
$\hat{w}_2$ & $0.600999\pm6.72\mathrm{e}{-4}$\\ 
$\hat{w}$ & $1.179325\pm3.41\mathrm{e}{-4}$\\ 
\hline
$R^2$ full & $0.825028\pm6\mathrm{e}{-6}$\\
$R^2$ aggr & $0.825319\pm5\mathrm{e}{-6}$\\
\text{MSE} full & $0.285611\pm9\mathrm{e}{-6}$\\
\text{MSE} aggr & $0.285137\pm8\mathrm{e}{-6}$\\
Var full & $0.001526\pm0.001557$\\
Var aggr & $0.000976\pm0.003587$\\
Bias full & $0.284086\pm0.046098$\\
Bias aggr & $0.284161\pm0.045194$\\
\hline 
Aggregations (theo) & $\mathbf{500}$ \\
Aggregations (emp) & $\mathbf{335}$ \\
\hline
\end{tabular}}
\end{table}

This subsection explains more details about the synthetic experiments performed in the three-dimensional setting and introduced in Section \ref{sec:appl} of the main paper. They show the usefulness of the extension to the three-dimensional linear regression model. In particular the samples of the first independent variable $x_1$ are extracted from a uniform distribution in the interval $[0,1]$. The second feature $x_2$ is a linear combination between the feature $x_1$ and a random sample extracted from a uniform distribution in the interval $[0,1]$ (specifically $x_2=0.65x_1+0.35u,\ u\sim\mathcal{U}([0,1])$). The third feature $x_3$ is a linear combination between the features $x_1,x_2$ and a random sample extracted from a uniform distribution in the interval $[0,1]$ ($x_3=0.5x_1+0.5x_2+0.5u,\ u\sim\mathcal{U}([0,1])$). Finally, the target variable $y$ is a linear combination between the three features $x_1,x_2,x_3$ with weights $w_1=0.4,\ w_2=0.6$ that are closer than the third weight $w_3=0.2$ and the addition of a gaussian noise with variance $\sigma^2=0.25$.\\
The experiment has been repeated $s=500$ times with $n=500$ samples both for the train and the test set. As reported in Table \ref{tab:3dimSynExtended}, which is an extension of Table \ref{tab:3dimSyn} reported in the main paper, the theoretical values of correlation thresholds computed from the asymptotic result of Equation \eqref{eq:bound3D} and the empirical ones computed substituting the unbiased estimators of the quantities show that it is convenient to aggregate the two features $x_1,x_2$. This is confirmed both by the $MSE$ and the $R^2$ coefficient, which are statistically not worse in the aggregated case than in the three dimensional one.

\subsection{D-dimensional synthetic data}\label{subs:NdimSynt}

This subsection provides more details on the application, introduced in Section \ref{sec:appl} of the main paper, of the algorithm \emph{\algnameshort} on a D-dimensional synthetic dataset. In particular, $D=100$ features are considered. The samples of the first independent variable $x_1$ are extracted from a uniform distribution in the interval $[0,1]$. Then, each feature $x_i$, is a linear combination between one of the previous features $x_j,\ j<i$ and a random sample extracted from a uniform distribution in the interval $[0,1]$ (specifically $x_i=0.7x_j+0.3u,\ u\sim\mathcal{U}([0,1])$). Finally, the target variable $y$ is a linear combination between the $D$ features $x_1,..,x_{100}$, with coefficients randomly sampled from a uniform distribution in the interval $[0,1]$, and a gaussian noise with standard deviation $\sigma=10$.\\
The algorithm is applied on the features both evaluating the threshold computed with the exact coefficients and with their unbiased estimates. The experiment has been repeated $s=500$ times on a dataset of $n=500$ samples both in train and test set, considering both the exact parameters (unkown in practice) and their estimators.

\subsection{D-Dimensional real datasets}\label{app:expReal}

This subsection describes the datasets introduced in Section \ref{sec:appl} of the main paper to apply the proposed algorithm \algnameshort on real data. \\
The first dataset considered focuses on the prediction of Life Expectancy from several factors that can be categorized into immunization related factors, mortality factors, economical factors and social factors. The dataset is available on Kaggle\footnote{https://www.kaggle.com/datasets/kumarajarshi/life-expectancy-who} and it is also provided in the repository of this work. It is made of $D=18$ continuous input variables and a scalar output.\\
The second dataset is a financial dataset made of $D=75$ continuous features and a scalar output. The model predicts the cash ratio depending on other metrics from which it is possible to derive many fundamental indicators. The dataset is taken from Kaggle\footnote{https://www.kaggle.com/datasets/dgawlik/nyse} and it is provided in the repository of this work.\\
Finally, the algorithm is tested on two climatological dataset composed by $D=136$ and $D=1991$ continuous climatological features and a scalar target which represents the state of vegetation of a basin of Po river. This datasets have been composed by the authors merging different sources for the vegetation index, temperature and precipitation over different basins (see~\citep{Didan2015,Cornes2018,Zellner2022}), and they are available in the repository of this work.

\end{appendices}

\end{document}